\documentclass[10pt, twocolumn]{article}

\usepackage[left=0.75in, right=0.75in, top=1in, bottom=1in]{geometry}

\usepackage{microtype}      
\usepackage{graphicx}       
\usepackage{booktabs}       
\usepackage{url}            
\usepackage[hidelinks]{hyperref} 

\usepackage{amsmath}
\usepackage{amssymb}
\usepackage{mathtools}
\usepackage{amsthm}

\usepackage[numbers,sort&compress]{natbib}

\usepackage{authblk}

\theoremstyle{plain}
\newtheorem{theorem}{Theorem}[section]
\newtheorem{proposition}[theorem]{Proposition}
\newtheorem{lemma}[theorem]{Lemma}
\newtheorem{corollary}[theorem]{Corollary}
\theoremstyle{definition}

\theoremstyle{remark}
\newtheorem{remark}[theorem]{Remark}

\title{\textbf{The Quantization Trap: Breaking Linear Scaling Laws in Multi-Hop Reasoning}}

\author[1]{Henry Han\thanks{Corresponding author: \texttt{henry\_han@baylor.edu}}}
\author[2]{Xiyang Liu}
\author[2]{Xiaodong Wang\thanks{Corresponding author: \texttt{wangxiaodong1021@gmail.com}}}
\author[3]{Fei Han}
\author[4]{Xiaodong Li}

\affil[1]{School of Engineering and Computer Science, Baylor University, Waco, TX 76798, USA}
\affil[2]{School of Computer Science and Technology, Xidian University, Xi'an, China, 710126}
\affil[3]{School of Computer Science and Communication Engineering, Jiangsu University, Zhenjiang, 212013, China}
\affil[4]{Beijing Electronic Science and Technology Institute, Beijing 100070, China}

\date{}

\begin{document}

\maketitle

\begin{abstract}
Neural scaling laws provide a predictable recipe for AI advancement: reducing numerical precision should linearly improve computational efficiency and energy profile ($E \propto \text{bits}$). In this paper, we demonstrate that this scaling law breaks in the context of multi-hop reasoning. We reveal a 'quantization trap' where reducing precision from 16-bit to 8/4-bit paradoxically increases more net energy consumption while degrading reasoning accuracy. We provide a rigorous theoretical decomposition that attributes this failure to hardware casting overhead, the hidden latency cost of dequantization kernels, which becomes a dominant bottleneck in sequential reasoning chains, as well as to a sequential energy amortization failure. As a result, scaling law breaking is unavoidable in practice. We formalize a Critical Model Scale $N^*$ that 
predicts when the trap dissolves or deepens as a 
function of model size, batch size, and hardware 
configuration, validated across a 120$\times$ 
range (0.6B--72B) on six GPU architectures. Our findings suggest that the industry’s “smaller-is-better” heuristic is mathematically counterproductive for complex reasoning tasks.
\end{abstract}

\vspace{0.5em}
\vspace{1em}


\section{Introduction}

The "Scaling Law" is the foundational dogma of modern Large Language Models (LLMs) \cite{kaplan2020scaling}. It posits that performance is a predictable function of compute and by extension, that quantization is almost a "free lunch" for efficiency: reducing bits $b$ should yield a proportional reduction in energy $E$ and memory $V$.

Consequently, the industry has widely adopted aggressive compression techniques to facilitate wider LLM deployment. For instance, the transition from 16-bit floating point (FP16) to 4-bit NormalFloat (NF4) or GPTQ-weighted quantization \cite{dettmers2024qlora, frantar2022gptq} has become the standard protocol for serving 70B-parameter models on single consumer-grade GPUs. These methods are built on the empirical success of \textit{'near-lossless'} quantization in general linguistic tasks, where the marginal increase in perplexity is viewed as a negligible price for the 4$\times$ reduction in memory footprint and the theoretical 2--4$\times$ gain in throughput \cite{dettmers2024qlora, lin2023awq} . This \textit{'linear scaling law'} paradigm assumes that efficiency gains are monotonic: as the bit-width $b$ decreases, the inference cost of LLM (functionally the 'cost-to-reason' in complex tasks) should follow a downward linear trajectory.

However, we identify a fundamental breakdown of this linear scaling paradigm within multi-hop reasoning. Multi-hop reasoning \cite{wei2022} is the process of decomposing a complex problem into a sequential chain of intermediate logical steps, where each deduction acts as a strict prerequisite for the next. Here, the sequential accumulation of quantization noise across logical chains triggers a violation of the scaling law and we term this as \textit{Quantization Trap}, where energy consumption increases and reasoning efficiency decreases for low-bit models, irrespective of model scale. 

Technically, the scaling law violation (the mechanism) drives the model into this trap (the outcome), where the compressed model is strictly dominated by the original baseline.

We attribute this failure to the hardware-level \textit{'Casting Overhead,'} where the cumulative micro-latency of de-quantizing weights back to native precision at each reasoning hop outweighs the memory bandwidth gains of compression besides rigorous analysis.
Our specific contributions:

\begin{enumerate}
    \item \textit{Sustainability Index (SI) framework:} We develop a Sustainability Index (SI) framework to evaluate AI across the three pillars of sustainability: Economic Efficiency ($E_{SI}$), Social Trust ($T_{SI}$), and Environmental Energy ($S_{SI}$) to detect the Quantization Trap, in a mathematically rigorous way, which is a sub-optimal AI configurations (\textit{Sustainability Inversion}).
    \item \textit{Quantization Trap in multi-hop reasoning:} We empirically identify and characterize a fundamental breaking of linear scaling laws in multi-hop reasoning. Benchmarking diverse models (e.g., Mistral-7B \cite{mistral}, Qwen-3 (0.6B)) on multi-hop   reasoning-heavy datasets (e.g., GSM8K and MathQA) \cite{gsm8k, mathqa}, across NVIDIA L4, A100, H100, and Blackwell RTX 6000-pro architectures.
    \item \textit{Theoretical foundations of the Quantization Trap:} We prove a \emph{Sequential Amortization Failure} theorem establishing that the Quantization Trap is structurally unavoidable in multi-hop reasoning, and demonstrate that standard mitigation strategies are mathematically incapable of restoring reasoning trust.
\end{enumerate}

By exposing the breakdown of scaling laws, our findings provide a rigorous mathematical bridge between hardware telemetry and cognitive performance, enabling developers to identify and avoid \textit{Quantization Traps} before committing expensive GPU resources. Ultimately, this discovery will force a fundamental shift in LLM development away from brute-force compression toward precision-aware scaling that is tailored to the actual complexity of the reasoning task, besides a mitigation path.

\section{Blind spots in LLM scaling laws}

The current understanding of LLM scaling is dominated by the power-law relationships established by \cite{kaplan2020scaling} and \cite{hoffmann2022chinchilla}, implicitly assuming that reducing numerical precision yields a guaranteed, linear reduction in hardware cost. While these  scaling laws provide a recipe for pre-training, they treat "efficiency" as a one-dimensional variable (typically FLOPs or memory footprint), effectively blinding practitioners to the non-linear latency and energy penalties incurred by hardware-level de-quantization.

Consequently, the subsequent explosion in quantization research \cite{dettmers2024qlora, frantar2022gptq} has focused almost exclusively on preserving model accuracy at lower bit-widths, assuming that hardware efficiency (e.g., on GPU) improves monotonically with bit-reduction. 

However, existing scaling laws and evaluation frameworks suffer from at least two critical blind spots when applied to \textit{multi-hop reasoning}:
\begin{enumerate}
    \item  \textit{Sequential Deductive Fragility:} They ignore the sequential dependencies of reasoning tasks, where quantization noise compounds at every reasoning step. This accumulation makes multi-hop reasoning uniquely fragile, as minor errors in early reasoning hops act as false premises for subsequent steps, causing irreversible logical divergence. This drives the model into the \textit{'Quantization Trap'}, which remains invisible to standard single-turn benchmarks \cite{hendrycks2020}. 
    \item   \textit{Hardware-Software Disconnect:} They ignore the \textit{'Casting Overhead'} \cite{zhao2024atom}, the physical energy cost of de-quantizing low-bit weights back into native precision (FP16/BF16) during execution. This disconnect is particularly acute on widespread architectures like \textit{NVIDIA Ampere (A100)} and \textit{Hopper (H100)}; while powerful, these architectures lack native arithmetic support for non-standard compressed formats (e.g., NF4), necessitating a mandatory, high-energy de-quantization step for every inference operation \cite{dettmers2024qlora, nvidia2020ampere}. As a result, \textit{the linear scaling intuition} (“fewer bits $\Rightarrow$ lower cost”) may not hold.  The extra de-quantization (“casting”) work can outweigh the bandwidth savings, so using fewer bits does not reliably reduce latency or energy, and can even increase them.
\end{enumerate}

\textit{Quantization-Aware Inference Optimization.}
Recent advancements mitigate low-bit hardware overheads through 
kernel-level optimizations (e.g., GPTQ \cite{frantar2022gptq}, 
Marlin \cite{marlin2024}) and weight-preservation strategies 
(AWQ \cite{lin2024awq}, QuIP\# \cite{tseng2024quip}). 
Concurrently, systems like FlexGen \cite{sheng2023flexgen} and 
Atom \cite{zhao2024atom} maximize efficiency for 
throughput-constrained deployments. While these methods 
demonstrably reduce the \textit{throughput} penalty, they do not 
address the \textit{accuracy} degradation caused by quantization 
noise compounding across sequential reasoning steps, a gap 
that motivates the present work.

\textit{Quantization and Reasoning Accuracy.} 
Recent studies highlight how quantization affects reasoning. Empirically, W8A8 and W4A16 maintain near-lossless accuracy, while lower bit-widths severely degrade mathematical reasoning---particularly in smaller models~\cite{liu2025quantization, li2025quantmeets}. On the efficiency side, aggressive quantization can paradoxically increase reasoning latency and energy costs~\cite{Ra2024, Hu2025}. While valuable, these empirical studies share two critical gaps that our work resolves:

\begin{itemize}
    \item \textit{Isolated Evaluation:} Prior work measures accuracy or efficiency, but never both jointly. As our AWQ experiments show (Section~\ref{sec:cross-backend}), accuracy-only metrics can declare W4A16 ``lossless'' while masking a 1.7$\times$ energy penalty. Our SI framework provides the joint diagnostic needed to expose these hidden costs.
    
    \item \textit{Lack of Theoretical Mechanisms:} Existing studies report empirical trends (e.g., ``model size and task difficulty matter'') without explaining \emph{why}. Our theorems (Sequential Amortization Failure, Amortization-Trust Decoupling, and Critical Model Scale $N^*$) provide a predictive foundation. For example, our $N^*$ formula (Eq.~\ref{eq:nstar}) formalizes why larger models resist degradation, while our per-hop analysis (Table~\ref{tab:per-hop}) proves that accuracy loss scales specifically with reasoning \emph{depth}, rather than monolithic task difficulty.
\end{itemize}

To bridge this gap, we introduce a sustainability-focused 
evaluation that goes beyond single-number accuracy or perplexity. 
Specifically, we propose a unified, quantitative framework that 
jointly measures an LLM's economic efficiency (throughput and 
deployability), trustworthiness (task performance under reasoning 
stress), and environmental cost (energy per query). This framework 
makes it possible to systematically detect the 
\textit{Quantization Trap}: the regime where apparent efficiency 
gains from aggressive low-bit quantization collapse under real 
deployment constraints, pushing the system into a sub-optimal 
configuration that is strictly dominated by its full-precision 
counterpart.

Section 3 introduces the framework of Sustainability Index (SI) to formally quantify these trade-offs. Figure~\ref{fig:fig_0} illustrates our end-to-end evaluation pipeline, from model quantization through multi-hop reasoning to Sustainability Index computation and trap detection.

\begin{figure*}[!htpb]
\centering
\includegraphics[width=0.70\textwidth]{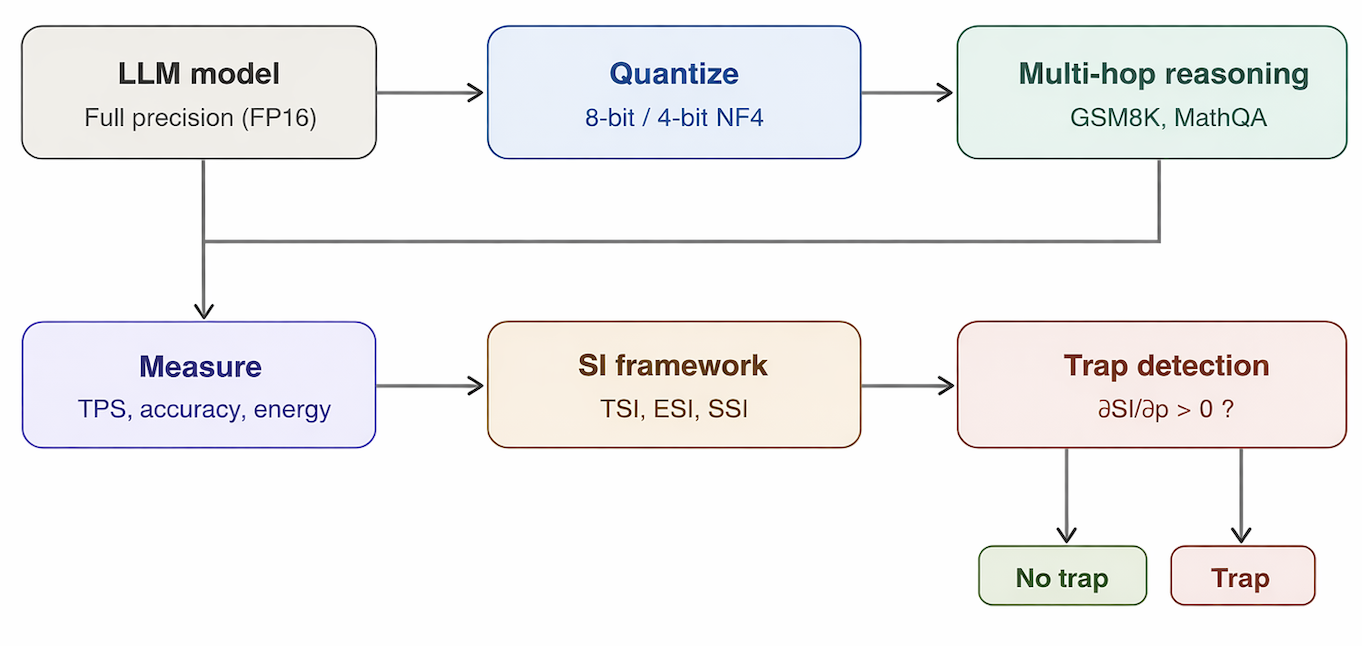}
\caption{Overview of the evaluation pipeline. An LLM is quantized to lower precision, evaluated on multi-hop reasoning benchmarks, and measured across three dimensions (throughput, accuracy, energy). The SI framework aggregates these into Trust, Economic, and Energy pillars to detect the Quantization Trap ($\partial \text{SI}/\partial p > 0$).}
\label{fig:fig_0}
\end{figure*}

\section{The Sustainability Index (SI) Framework}

\textit{The Calculus of Sustainable AI:} To transcend qualitative heuristics and systematically detect structural anomalies such as the \textit{Quantization Trap}, we formalize the evaluation of an AI system as a mapping $\mathcal{M}: \Theta \to \mathcal{V}$. Here, $\Theta$ denotes the manifold of architectural configurations (e.g.,
hardware $h$,
precision (bit-width) $p$, etc.) and $\mathcal{V} \subset [0, 1]^3$ is a bounded sustainability vector space. For any configuration $\theta \in \Theta$, we define the \textit{Sustainability Vector} $\mathbf{v}(\theta)$ as:
\begin{equation}
    \mathbf{v}(\theta) = \big[ T(\theta), \; E(\theta), \; S(\theta) \big]^\top
\end{equation}
representing the orthogonal dimensions of \textit{Trust}, \textit{Efficiency}, and \textit{Energy}, respectively. Each dimension is normalized against a reference anchor $\theta_{ref}$ (typically the full-precision baseline) to ensure scale-invariance across heterogeneous hardware.

\subsection{Mathematical Axiomatization of the Pillars}

\paragraph{I. The Trust Sustainability ($T_{SI}$).} 
In our framework, Trust is not a monolithic attribute but a \textit{context-dependent manifold} defined by the functional requirements of a specific AI task topology $\psi$. We define a set of normative metrics $\mathcal{M} = \{m_1, m_2, \dots, m_k\}$, where the relevant metric $m \in \mathcal{M}$ is selected based on the risk profile of the domain. The Trust index $T_{SI}(\theta)$ is the preservation ratio of the primary metric relative to the anchor:
\begin{equation}
    T_{SI}(\theta) = \frac{ \mathcal{G} \big( \mathbf{m}(\theta) \big) }{ \mathcal{G} \big( \mathbf{m}(\theta_{\text{ref}}) \big) } \in [0, 1]
\end{equation}
where $\mathbf{m}$ is the vector of observed performance metrics and $\mathcal{G}(\cdot)$ is a task-specific aggregation operator (e.g., accuracy, fairness-constrained utility) determined by the task topology $\psi$.

\textit{Application to Reasoning:} In logical synthesis tasks (e.g., GSM8K), we define $\mathcal{G}(\cdot)$ strictly as the reasoning accuracy. This constraint reflects the domain requirement that trust is contingent on the model's ability to maintain deductive consistency and arrive at a verifiable ground-truth solution.
    
\textit{Example:} Evaluating Mistral-7B on an A100 GPU, the FP16 anchor achieves an accuracy of $0.4314$, while the 4-bit (NF4) variant yields $0.3942$. This results in a trust score of $T(\theta) = 0.3942 / 0.4314 = 0.914$, quantifying an $8.6\%$ loss in deductive fidelity.

\paragraph{II. Economic Sustainability ($E_{SI}$)} 
We quantify economic sustainability by aggregating efficiency at two distinct scales: (i) \textit{Local Algorithmic Efficiency} (measured in Tokens Per Second, TPS), which rewards the model's raw computational speed; and (ii) \textit{Global Resilience} (measured in Peak VRAM, $V_{\text{peak}}$), which penalizes high hardware barriers to ensure deployment viability.

 $E_{SI}$ is defined as the Weighted Harmonic Mean of the normalized metrics. Let $\eta = \text{TPS}(\theta)/\text{TPS}_{\text{ref}}$ be the throughput gain and $\rho = V_{\text{ref}}/V_{\text{peak}}(\theta)$ be the residency gain. The index is defined as:
\begin{equation}
    E_{SI}(\theta) = \left( \frac{\alpha}{\eta} + \frac{1-\alpha}{\rho} \right)^{-1}, \quad \alpha \in [0,1]
\end{equation} where $\alpha=0.5$ by default.
Unlike arithmetic or geometric means, the harmonic formulation is \textit{bottleneck-sensitive}. It strictly penalizes deficiencies in either dimension, preventing 'extreme-case bias' where a model might achieve high throughput simply by consuming excessive memory (or vice-versa). By explicitly including $\rho$, our economic sustainability index  penalizes the systemic \textit{'compute divide'} \cite{ahmed2020compute_divide}.

\paragraph{III. The Energy Pillar ($S_{SI}$).} 
We define the environmental cost as the integrated power draw over the inference duration, mapped to its carbon-equivalent impact. Let $P(t)$ be the instantaneous power consumption (in Watts) tracked via hardware-level telemetry (e.g., NVML \cite{nvml}). The energy-per-query $\mathcal{E}_q$ is the temporal integral of power divided by the sample count $N$:
\begin{equation}
    \mathcal{E}_q(\theta) = \frac{1}{N} \int_{0}^{\mathcal{D}} P(t) \, dt \approx \frac{\sum_{t=1}^{\mathcal{D}} P_t \cdot \Delta t}{N}
\end{equation}
where $\mathcal{D}$ is the total execution latency. To account for geographical variability in energy production, we define the \textit{Carbon-Adjusted Energy Score (CAES)}, denoted as $
    \chi(\theta) = \mathcal{E}_q(\theta) \times \gamma $, where $\gamma$ is the grid carbon intensity ($\text{kgCO}_2\text{e}/\text{kWh}$ \cite{schmidt2021}).

The Energy Sustainability index $S_{SI}(\theta)$ is the inverse normalized ratio of CAES relative to the anchor:
{\small
\begin{equation}
S_{SI}(\theta)
= \min\left(1,\;
\frac{\log\!\bigl(1+\chi(\theta_{\text{ref}})\bigr)}
     {\log\!\bigl(1+\chi(\theta)\bigr)}
\right).
\end{equation}}

\textit{Energy-per-Query Interpretation.} This formulation ensures that smaller values of $S_{SI}(\theta)$ denote an \textit{environmental sustainability deficit} (higher emissions). Notably, if models are evaluated within the same grid infrastructure, the carbon intensity $\gamma$ becomes a constant scalar that cancels out during normalization. We omit the logarithmic transform in this context because environmental impact accumulates linearly; a linear ratio preserves sensitivity to extreme energy spikes (e.g., distinguishing $10\times$ from $100\times$ deficits) that logarithmic scaling would misleadingly compress. In such cases, the index reduces to a pure energy-per-query ratio ($\mathcal{E}_{\text{ref}} / \mathcal{E}_{\theta}$), allowing Joules per Query ($J/Q$) to serve as a direct proxy for environmental impact.

\textit{$S_{SI}$ Example (Same grid).} With the FP16 model  as the anchor with energy per query: $E_{\text{anchor}}=519.79$ J/query, the 4-bit (NF4) model has $E_{\text{model}}=1110.56$ J/query, giving $S_{SI}=0.468$, suggesting worse energy emission

\textit{Energy calculation for multi-hop reasoning.} We quantify $S_{SI}$ via Hardware-Latency Integration (HLI): $\mathcal{E}_q = (\mathcal{P}_{\text{TDP}} \cdot \mathcal{D}) / N$, where $\mathcal{P}_{\text{TDP}}$ is the architectural power anchor and $\mathcal{D}$ is high-resolution latency and $N$ is the sample size in the duration. Unlike NVML, which underestimates energy by misinterpreting sequential logic-dispatch micro-stalls as "idle savings" \cite{huang2020evaluating, henderson2020towards}, HLI treats the GPU as a committed resource throughout the reasoning duration. Because multi-hop reasoning’s high-frequency 'sawtooth' power profile exceeds NVML’s sampling frequency \cite{kumbhare2020gpu}, stochastic sensors smear transient, sub-millisecond de-quantization spikes into inter-token gaps. By integrating peak power over total latency, HLI accurately captures the energy deficit driven by casting overhead, exposing the true depth of energy use in multi-hop reasoning.

\subsection{The Sustainability Manifold: Aggregation and Scaling Laws}

We define the global \textit{Sustainability Index (SI)} as the scalar projection of the vector $\mathbf{v}(\theta)$ through a policy function $\mathcal{F}: \mathbb{R}^3 \to [0, 1]$. While traditional scaling laws focus on a 1D trajectory (e.g., bits vs. loss), our framework evaluates model configurations as points on a 3D \textit{Sustainability Manifold}.

In this work, we instantiate $\mathcal{F}$ as a \textit{Compensatory (Linear) Policy}, defined as:

\begin{equation}
    SI(\theta) = \mathbf{w}^\top \mathbf{v}(\theta), \quad \text{where } \sum_{i \in \{T, E, S\}} w_i = 1
\end{equation}

This formulation enables us to quantify the marginal rate of substitution between Economic efficiency, Energy, and Trust within the global sustainability score. While the weight vector $\mathbf{w}$ can be tuned to reflect specific task priorities, we adopt a balanced configuration for reasoning tasks ($w_T=0.34, w_E=0.33, w_S=0.33$) to ensure equal prioritization of all three pillars.

\begin{remark} While we employ a linear formulation to map the reasoning trade-off manifold, we acknowledge that safety-critical domains may necessitate a bottleneck-intolerant geometric aggregation ($SI = \prod v_i^{w_i}$) to prevent efficiency gains from masking catastrophic failure.

We note, however, 
that the detection of the Quantization Trap is \emph{robust} 
to weight selection: because FP16 Pareto-dominates all 
quantized configurations ($T_{SI}$, $E_{SI}$, and $S_{SI}$ 
are simultaneously higher), any convex combination with 
$w_i > 0$ preserves the same ranking. In deployment-specific 
contexts, practitioners may adjust weights to reflect 
operational priorities (e.g., $w_T \gg w_E$ for 
safety-critical reasoning, $w_S \gg w_T$ for 
energy-constrained edge deployment) without altering 
the trap diagnosis.
\end{remark}

\paragraph{The Scaling Monotonicity Hypothesis}
The prevailing scaling paradigm relies on the \textit{Monotonicity Hypothesis}: the assumption that reducing precision $p$ yields a monotonic increase in the scalar utility of the AI system, i.e., $\frac{\partial SI}{\partial p} < 0$. 

In multi-hop reasoning, however, reducing $p$ causes the Trust ($T_{SI}$) and Energy ($S_{SI}$) indices to degrade (see the next Section). This forces their partial derivatives with respect to precision to be positive ($\frac{\partial T_{SI}}{\partial p}, \frac{\partial S_{SI}}{\partial p} > 0$). When these penalties outweigh the marginal economic gains of compression ($\frac{\partial E_{SI}}{\partial p} < 0$), the global gradient $\nabla_{\!p} SI$ flips sign, mathematically detecting the \textit{Scaling Law Divergence} inherent in a Quantization Trap.

\begin{proposition}[Scaling Law Divergence]
Let $SI(\theta)$ be a function of precision (bit-width) $p$ for a configuration $\theta \in \Theta$. A \textit{Quantization Trap} is identified when: $
    \frac{\partial SI}{\partial p} > 0 
$ signifying a fundamental breakdown of the linear scaling law. (\textit{proof in appendix})
\label{prop:slb}
\end{proposition}

\begin{figure*}[t]
    \centering    \includegraphics[width=0.99\textwidth]{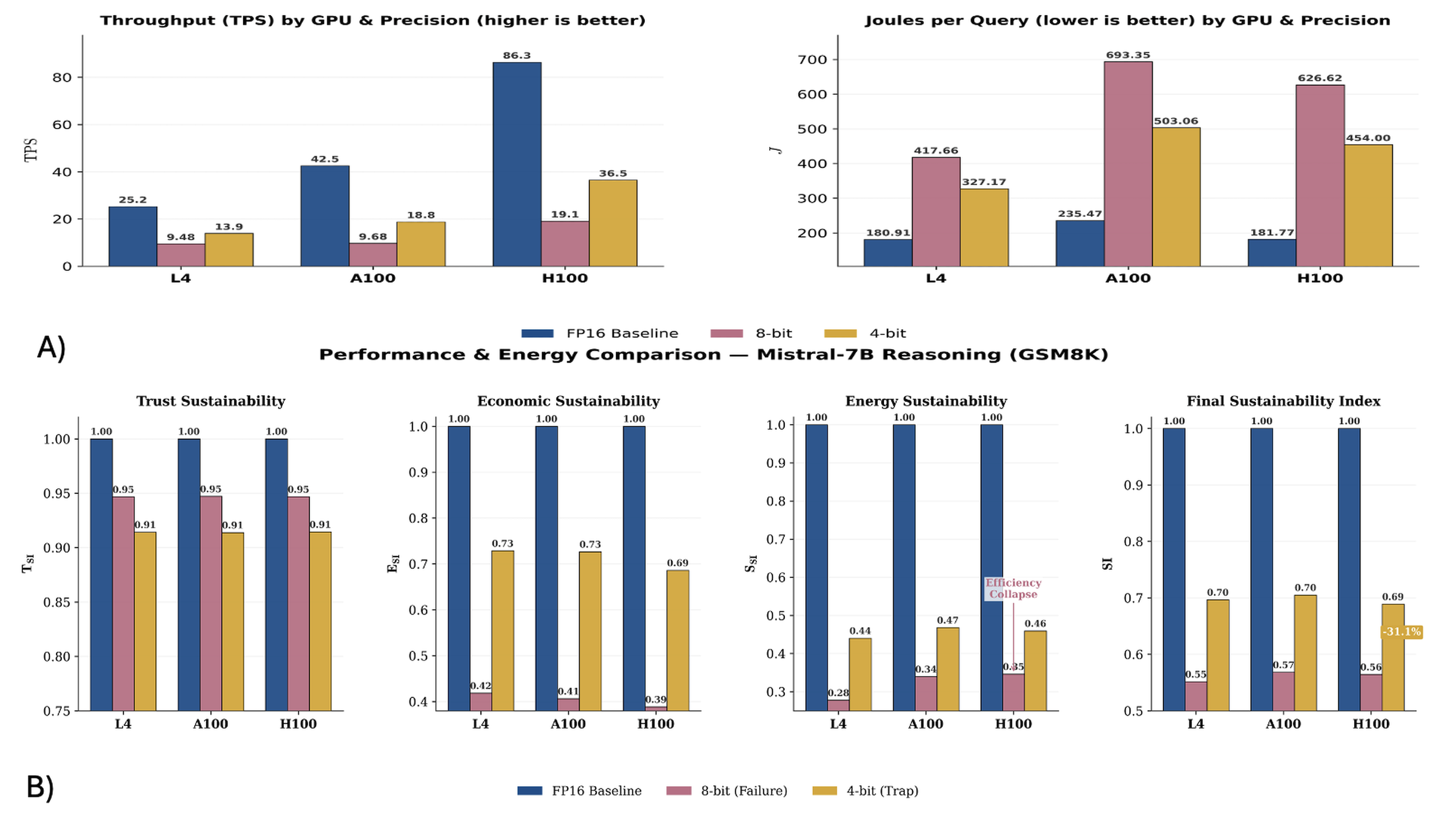}
    \caption{Sustainability Inversion (Mistral-7B, GSM8K).(A) Physical Telemetry: Reducing precision to 8/4-bit across L4/A100/H100 triggers non-monotonic throughput collapse and energy spikes; FP16 is more efficient due to software-emulated casting overhead. (B) Sustainability Manifold: Indices ($T_{SI}, E_{SI}, S_{SI}$) show low-bit configurations are Pareto-dominated. 8-bit represents systemic failure, while 4-bit marks a "Quantization Trap" with a 31.1\% global SI deficit relative to the FP16 anchor}
    \label{fig:fig1}
\end{figure*}

\section{Scaling Law Breaking in reasoning}

In this section,  we demonstrate that for multi-hopping reasoning tasks, the relationship between bit-precision and efficiency and energy is not merely sub-linear, but non-monotonic, leading to a structural 'Sustainability Inversion': the scaling law is  broken in a Quantization Trap.

\subsection{Multi-Hop Reasoning and casting overhead}

Multi-hop reasoning is distinguished from standard autoregressive generation by its requirement for \textit{deductive fidelity} across a sequence of $K$ intermediate logical states, $\{h_1, h_2, \dots, h_K\}$. In this regime, each reasoning "hop" $h_k$ is a logical precursor to $h_{k+1}$, creating a dependency chain where the probability of reaching the correct ground-truth solution $y$ is the joint product of successful intermediate transitions:
\begin{equation}
    P(y \mid x, \theta) = \prod_{k=1}^{K} P(h_k \mid h_{<k}, x, \theta)
\end{equation}

This sequential structure makes reasoning uniquely sensitive to Scaling Law Divergence through two primary mechanisms under quantization.

\textit{Deductive Fragility.} Quantization errors compound across sequential hops, where a minor shift in an early premise amplifies into "deductive drift" and logical divergence. This may cause reasoning Trust to decrease more rapidly than in single-turn tasks, driving the positive trust gradient $\frac{\partial T_{SI}}{\partial p} > 0$.

\textit{Temporal Accumulation.}
Total latency $\mathcal{D}_{\text{total}} \approx \sum_{k=1}^{K} (\tau_{\text{comp}} + \tau_{\text{cast}})$ aggregates native compute time $\tau_{\text{comp}}$ and de-quantization overhead $\tau_{\text{cast}}$ across all $K$ sequential logical hops. Here, a \textit{hop} is the atomic autoregressive forward pass (token generation). This cumulative penalty bottlenecks throughput and inflates power draw, systematically degrading both the economic ($E_{SI}$) and energy ($S_{SI}$) sustainability pillars.

\paragraph{Logical vs. Atomic Hops.} 
While the literature \cite{cobbe2021} often defines "multi-hop reasoning" by its $M$ high-level logical milestones (e.g., 8–9 arithmetic operations), these milestones are physically materialized on hardware through a much longer sequence of $K$ \textit{atomic hops} (autoregressive token generations), where $K \gg M$, specifically in the context of multi-hop reasoning \cite{pope2023}. We anchor our analysis at the \textit{atomic hop} level, because it is the fundamental unit of hardware execution where the casting overhead $\tau_{\text{cast}}$ is incurred and compounded.

\paragraph{Casting Overhead Ratio (COR):}  We formalize the hardware bottleneck via defining the \textit{Casting Overhead Ratio ($COR$)}"  the ratio of de-quantization latency to native compute time:
\begin{equation}
    COR = \frac{\tau_{\text{cast}}}{\tau_{\text{comp}}}
\end{equation}
The $COR$ serves as a diagnostic for \textit{Casting Overhead Dominance}; when $COR > 1.0$, the de-quantization "tax" outweighs the actual computational work. We have following lemma and proposition related to the $COR$ (\textit{proof in appendix}).

\begin{lemma}[Average Latency per hop]
Let $TPS$ be the tokens per second. The average latency per hop is $\tau = 1/TPS$.
\end{lemma}

{\small\begin{proposition}[$COR$ Approximation]
Given a high-precision reference model (ref) and a quantized model ($p$), the Casting Overhead Ratio is given by:
\begin{equation}
    COR \approx \frac{TPS_{\text{ref}}}{TPS_p} - 1
\end{equation}
\end{proposition}}

\subsection{Evaluation Models, Data and Hardware} 
To probe the limits of scaling laws, we evaluate two primary reasoning-optimized models: \textit{Mistral-7B-v0.3}, a high-performance model with strong deductive capabilities, and \textit{Qwen-3-0.6B}, an ultra-thin architecture designed for extreme resource-constrained deployment. We further corroborate our findings using \textit{Falcon-3B} \cite{falcon3} to ensure that the observed \textit{Quantization Trap} is a consistent phenomenon across mid-tier parameter scales.

\textit{GSM8K and MathQA \cite{mathqa} benchmarks:} These models are tested on \textit{GSM8K} (7.5K training / 1K test samples), consisting of multi-step grade-school arithmetic, and \textit{MathQA} (29.8K training / 2.9K test samples), which focuses on complex algebraic transformationsc. 

GSM8K consists of 8.5K high-quality, multi-step elementary school math word problems. Unlike simple question-answering, GSM8K requires models to generate a \textit{Chain-of-Thought} (CoT) involving 2 to 8 arithmetic reasoning steps to arrive at a final verifiable answer. It serves as our primary probe for \textit{arithmetic noise propagation}, where a single precision-induced error in an early intermediate step (e.g., a carry-over mistake) inevitably causes the entire logical chain to diverge from the ground truth.

MathQA demands significantly higher deductive depth in reasoning. MathQA stress-tests \textit{Scaling Law Divergence} through complex algebraic transformations.  Its extended logical chains amplify both Casting Overhead accumulation and the compounding effects of \textit{quantization noise}.

\textit{Quantization Protocol.}
All experiments employ \emph{weight-only} quantization 
with activations retained in FP16 (W4A16 and W8A16). 
Specifically, we use bitsandbytes~\cite{dettmers2024qlora} 
NF4 with double quantization for 4-bit and LLM.int8() 
for 8-bit inference. This targets the memory-bound 
autoregressive decoding regime dominant in reasoning 
tasks at low batch sizes ($B \approx 1$). We note that 
weight-activation quantization (e.g., W4A4, W8A8) 
addresses a distinct, compute-bound serving regime 
and is outside the scope of this work. Our SI framework 
is method-agnostic and can evaluate any quantization 
backend; bitsandbytes was selected as the most widely 
deployed library in practice (default in HuggingFace 
Transformers).

\textit{Hardware Instrumentation.} We mainly conduct our experiments on NVIDIA L4, A100, and H100 GPUs, spanning the Ada Lovelace, Ampere, and Hopper architectures, respectively. These platforms represent the dominant compute paths in contemporary AI deployment, ensuring that our characterization of Scaling Law Divergence reflects the physical reality of modern hardware constraints.

\subsection{Sustainability Inversion: the Quantization
Trap}

\paragraph{Mistral-7B Reasoning on GSM8K.}
Figure 1 summarizes the physical telemetry and sustainability analysis for Mistral-7B on GSM8K across L4, A100, and H100 GPUs. These results reveal a distinct \textit{Quantization Trap}, where the linear scaling law is fundamentally broken.

\begin{figure*}[!htpb]
    \centering
\includegraphics[width=0.99\linewidth]{fig_2.png}
     \caption{\small The Size Paradox (Qwen3-0.6B). (A) 4-bit quantization triggers universal reasoning collapse regardless of architecture. (B) High 8-bit COR values (2.5–2.8$\times$) prove casting overhead dominance is the mechanical driver of inefficiency. (C) Telemetry reveals that low-precision "optimization" paradoxically results in a 400\% energy penalty}
    \label{fig:fig2}
\end{figure*}

\textit{Physical Proof of Scaling Law Divergence.} 
The empirical results in Figure 1(B) provide a direct physical validation of Proposition \ref{prop:slb}. While the traditional scaling law assumes a negative gradient ($\frac{\partial SI}{\partial p} < 0$), where reducing precision should yield higher utility, the Final Sustainability Index (SI) exhibits a monotonic increase as precision $p$ is restored from 8-bit to 16-bit. Across all three architectures (L4, A100, H100), we observe $SI(\text{FP16}) > SI(\text{4-bit}) > SI(\text{8-bit})$. On the H100, for instance, restoring precision from 4-bit to 16-bit yields a 31.1\% improvement in global sustainability. This positive gradient, $\frac{\partial SI}{\partial p} > 0$, signifies that for multi-hop reasoning, the "smaller is better" hypothesis is not only sub-optimal but mathematically inverted; increasing precision is the only path to a more efficient and sustainable system.

\textit{The Energy Paradox ($S_{SI}$).} Under standard scaling laws, energy consumption is expected to decrease with bit-precision ($J/\text{query} \propto p$). However, our telemetry reveals a massive \textit{sharp energy increase} in low-bit models. On the A100, the FP16 baseline represents the \textit{Efficiency Peak} (235.47~J), while transitioning to 8-bit and 4-bit (NF4) increases energy consumption to 693.35~J and 503.06~J, respectively. This environmental deficit is most severe on the H100, where the 8-bit energy index $S_{SI}$ drops to 0.35, representing a critical deficit in sustainability.

\textit{Decay of Deductive Trust ($T_{SI}$).} We observe a steady decay in the Trust pillar as precision decreases, reflecting the "deductive drift" inherent in multi-hop chains. On the A100, reasoning accuracy falls from 43.14\% (FP16) to 40.86\% (8-bit) and 39.42\% (4-bit). While the logic remains semi-coherent, the compounding quantization noise ensures that low-bit models are less reliable for reasoning tasks.

\textit{Economic Viability Decrease ($E_{SI}$).} Contrary to expectations, reducing precision triggers a non-monotonic throughput collapse across all architectures (Figure 1(A)). This occurs because the hardware-level \textit{Casting Overhead Ratio} (COR) becomes the primary execution bottleneck; the micro-latency of de-quantizing weights outweighs the theoretical gains in memory bandwidth. Consequently, the Economic index $E_{SI}$ for 8-bit configurations falls to $\approx 0.40$, signaling a total collapse in operational scalability for reasoning-heavy workloads.

\subsubsection{ The Size Paradox: Ultra-Thin Models}

One might dismiss the \textit{Scaling law breaking in Quantization Trap} as a byproduct of the massive memory bandwidth requirements inherent to 7B-scale models; however, by evaluating the ultra-thin Qwen-3 (0.6B), we reveal a \textit{Size Paradox}: the trap is not merely a symptom of model scale, but an inescapable architectural law that becomes structurally deeper as the parameter footprint shrinks.

As illustrated in Figure ~\ref{fig:fig2}, the 0.6B model provides the definitive evidence for Scaling Law Divergence on the GSM8K data. In Panel A, we observe a terminal \textit{Reasoning Collapse} at 4-bit precision; while 8-bit maintains trust, the 4-bit variant exhibits a catastrophic drop in reasoning accuracy across all hardware. 

The mechanical cause is isolated in Panel B. At the 0.6B parameter scale, $\tau_{\text{comp}}$ becomes negligible. At the 0.6B scale, the actual matrix math is so minimal that it finishes within the GPU's fixed latency floor for kernel dispatch \citep{pope2023}. So the de-quantization overhead ($\tau_{\text{cast}}$) becomes the absolute dominant factor in the execution pipeline \citep{zhao2024atom} On the A100, the \textit{Casting Overhead Ratio ($COR$)} for 8-bit logic reaches an astounding 2.74, meaning the system spends nearly $3\times$ more energy emulating precision than performing deduction. 
 
 This results in the energy anomaly seen in Panel C, where the 8-bit model consumes over $1700$J per query: a 400\% increase over the FP16 baseline. These results prove that for small-scale reasoners (e.g., Qwen3-0.6B), the industry-standard "smaller-is-better" heuristic is not just sub-optimal, but \textit{counterproductive.}

\begin{figure*}[t]
    \centering
    \includegraphics[width=0.90\textwidth]{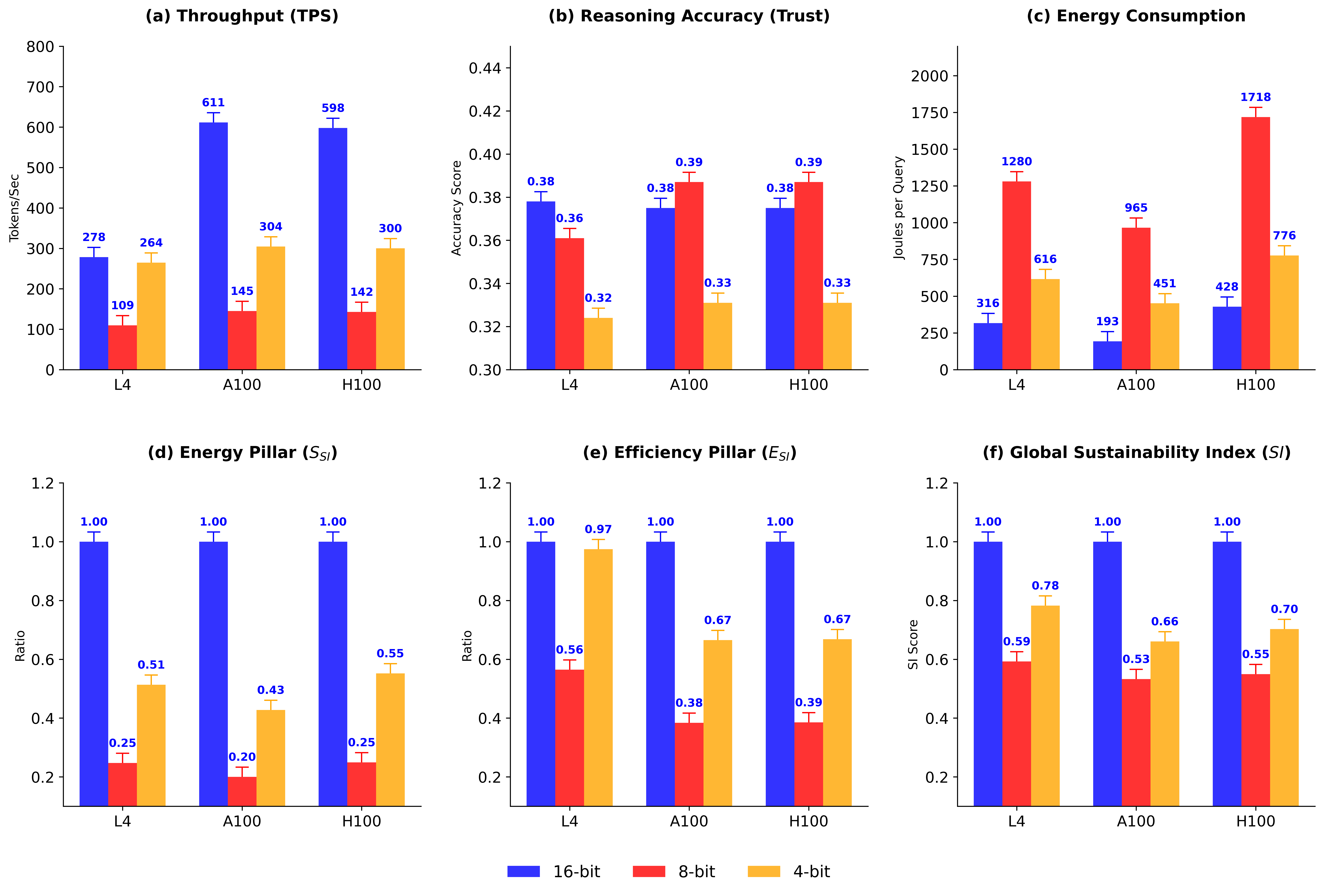}
    \caption{\small \textit{Sustainability Inversion of Qwen3-0.6B reasoning on MathQA}(a–c) Physical telemetry reveals throughput collapse and energy spikes; FP16 is $4.2\times$ more efficient than 8-bit on H100. (d–f) Sustainability indices expose systemic failure ($SI \approx 0.55$) in quantized models due to unamortized casting overhead. FP16 strictly Pareto-dominates all configurations across L4, A100, and H100, proving brute-force bit-reduction constitutes a Quantization Trap}
    \label{fig:fig3}
\end{figure*}

\textbf{Scaling law breaking of Qwen-3 (0.6B) reasoning on MathQA data.}To verify if higher logical complexity further compounds the hardware-level tax, we transition the ultra-thin Qwen-3 (0.6B) to the MathQA benchmark, subjecting its minimal parameter footprint to significantly deeper deductive chains.

\textit{Hardware Latency Dominance.} As illustrated in Figure ~\ref{fig:fig3}, the throughput collapse remains pervasive even at the sub-1B scale. On the A100, the 16-bit baseline achieves 611 TPS, whereas the 8-bit and 4-bit configurations are throttled to 145 and 304 TPS, respectively. This confirms that for ultra-thin reasoners, the "actual math" executes faster than the hardware's irreducible kernel-launch latency \cite{pope2023}, leaving the serial de-quantization instructions as the absolute bottleneck and forcing $COR > 1.0$.

\textit{The Trust Divergence.}
On the MathQA benchmark, which demands rigorous algebraic transformations, the 4-bit variant exhibits a terminal logical divergence. While 8-bit precision maintains reasoning fidelity, the 4-bit accuracy collapses from 0.38 to 0.33 across all GPU architectures (Figure 3b). This 13\% relative decay in accuracy proves that aggressive quantization is fatal for complex multi-hop tasks, where minor precision-induced errors in early premises inevitably derail long deductive paths.

\textit{Terminal Sustainability Inversion.}
Most critically, the Energy Paradox reaches its terminal phase on high-end architectures. On the H100 (Figure 3c), the 8-bit model consumes 1718J per query, over 4$\times$ the energy footprint of the 16-bit baseline (428J). Consequently, the Global Sustainability Index ($SI$) for all low-bit configurations remains strictly below 0.80 (Figure 3f). This provides the definitive proof of the \textit{Quantization Trap}: for ultra-thin models, compression is a mathematically irrational strategy that consumes more power to deliver strictly inferior logical output.

\begin{figure*}[t]
    \centering    \includegraphics[width=0.95\textwidth]{fig_4aa.png}
    \caption{\small \textit{Cross-Architectural Trap Evidence.} (A) Throughput analysis locates $B^* \approx 64$ for Falcon3, while Mistral-7B remains terminally trapped ($B^* > 128$). (B) $COR > 1.0$ identifies \textit{Casting Dominance} as the mechanical bottleneck. (C) A scale-invariant $\sim$30\% logic collapse proves that batch-driven efficiency gains fail to restore reasoning trust.}
    \label{fig:fig4}
\end{figure*}

\subsection{The scaling law breaking Driver: Amortization Failure}

The scaling law breaking in Quantization Trap is a mechanistic consequence of Amortization Failure within the GPU’s execution pipeline. \textit{On architectures (L4/A100/H100), Tensor Cores lack a native hardware logic chain for sub-8-bit arithmetic \cite{frantar2022gptq, lin2023awq}}, necessitating the de-quantization of weights back into native precision ($FP16$) before any mathematical operation can occur. 

This architectural void creates a high \textit{Casting Overhead Ratio (COR).} For Mistral-7B reasoning on GSM8K, the COR on H100 hardware reaches a staggering 2.45 for 8-bit precision and 1.50 for 4-bit precision. This implies that for every joule spent on actual inference, the system wastes up to 2.45 joules simply translating bits.


In high-concurrency scenarios, this overhead is "amortized" across the batch. However, multi-hop reasoning is operationally forced into a $B=1$ regime to satisfy sequential logic dependencies. In this state, the COR is never shared. While 4-bit (NF4) achieves a 75\% bandwidth benefit, the fixed COR ensures that the cumulative casting energy $\Phi = \sum_{k=1}^K COR$ grows linearly with the reasoning depth $K$, eventually outpacing all memory savings. We formalize this breakdown through the following theorem (\textit{proof in Appendix}).
{\small\begin{theorem}[Sequential Amortization Failure in multi-hop reasoning]
Let $\mathcal{E}(p, B)$ be the energy-per-query as a function of precision $p$ and batch size $B$. For a multi-hop reasoning task on hardware with native precision $\pi > p$, there exists a \textit{Critical Batch Threshold} $B^*$ such that for all $B < B^*$, the system enters a \textit{Quantization Trap} defined by: $\frac{\partial \mathcal{E}}{\partial p} < 0.$

This inversion occurs when the marginal casting overhead $\phi$ outpaces the amortized bandwidth savings, where $B^*$ is uniquely determined by:
\begin{equation}
    B^* = \alpha \frac{(\pi - p)}{\phi(h, p)}
\end{equation}
Here, $\alpha$ is the hardware energy constant for memory movement and $\phi(h, p)$ is the de-quantization cost for $p$ on architecture $h$.
\label{thm:Bstar}
\end{theorem}}

\textit{The Inevitability of the Quantization Trap.} Theorem~\ref{thm:Bstar} mathematically guarantees the Inevitability of the Quantization Trap when batch size $B<B^{*}$ the Critical Batch Threshold. Under the situation, the multi-hop reasoning is structurally forced into the low-batch regime by two opposing mandates.

1.  VRAM Residency (Large Models): Frontier models consume the majority of available GPU memory for weight storage. To maintain the memory-intensive cache required for long reasoning chains without triggering out-of-memory errors, the batch size must be kept at a minimum.

2.  Latency Mandates (Small Models): Compact models are primarily deployed in interactive, real-time contexts. Waiting to accumulate a large batch would violate the strict latency requirements of the agent, forcing the system to process requests individually.

These constraints prevent the amortization of hardware casting overhead, ensuring that reasoning caught in \textit{Quantization Trap.} Therefore, the empirical 'Sustainability Inversion' observed is not a stochastic anomaly, but a \textit{Systemic AI Failure}, a deterministic breakdown of the scaling law.

\textit{Scaling Law Divergence in multi-hop reasoning:} Building on the physical constraints of the critical batch threshold $B^*$, we synthesize the hardware energy deficit with the compounding requirements of logical trust to formally prove that Scaling Law Divergence is a structural necessity in multi-hop reasoning in the following theorem (\textit{proof in Appendix}).

\begin{theorem}[Scaling Law Divergence in Multi-Hop Reasoning]
For multi-hop reasoning tasks on hardware with native precision $\pi > p$, the Scaling Law Divergence ($\frac{\partial SI}{\partial p} > 0$) is structurally unavoidable for batch sizes $B < B^*$.
\label{thm:reasoning_divergence}
\end{theorem}

\subsection{ Amortization  and Quantization Trap Mitigation}

Theorem \ref{thm:Bstar} provides an increasing batch size way to decrease the casting overhead ratios. To mitigate the effects of the Quantization Trap, we investigate the scaling of batch size ($B$) as a mechanism to amortize the hardware-level tax. We subject Falcon3-3B (Base and Instruct)\cite{falcon3} and Mistral-7B architectures to increasing batch workloads to locate the Critical Batch Threshold ($B^*$) and analyze the recovery dynamics of the sustainability pillars on Blackwell RTX6000-pro.

Figure ~\ref{fig:fig4} illustrates the physical and reasoning response to batch scaling across three distinct architectures. The results reveal a fundamental tension between hardware throughput and logical fidelity:

\textit{Casting Dominance and Amortization.} Figure ~\ref{fig:fig4} Column B identifies the Casting Overhead Ratio ($COR$) as the Quantization trap's primary mechanical driver. At real-time scales ($B \le 16$), all models reside in the \textit{Casting Dominance} regime ($COR > 1.0$), where the GPU spends more energy emulating precision formats than executing logical deduction. While increasing $B$ amortizes this hardware tax per Theorem~\ref{thm:Bstar}, Column A reveals that 8-bit configurations remain terminally slower than 16-bit baselines even at $B=128$.

\textit{The $B^*$ Disparity.} The threshold $B^*$ marks the "sustainability equator" where quantization finally yields net gains. While Falcon3-3B reaches crossover at $B^* \approx 64$ (\textit{Recoverable Trap}), the larger Mistral-7B remains terminally trapped ($B^* > 128$). This scale-dependent disparity proves that de-quantization overhead becomes more persistent as parameter counts grow, fundamentally breaking the scaling law for workstation-class workloads. 

\textit{The Half-Mitigation Paradox.} Figure ~\ref{fig:fig4} Column C reveals that while batching repairs throughput and energy deficits, it is functionally incapable of restoring logical fidelity. For Falcon3-3B-Instruct, the 30\% logic collapse remains stagnant across all batch sizes because quantization noise $\epsilon$ is a per-token stochastic penalty that cannot be amortized. This identifies a terminal paradox: batching may make a quantized model "fast," but it cannot make it "accurate," rendering aggressive compression a failed strategy for high-fidelity reasoning. We have the following theorem to address this.

\begin{theorem}[Amortization-Trust Decoupling]
Let $\mathcal{E}(p, B)$ denote the energy-per-query and $T(p, B)$ denote the reasoning accuracy (Trust) for a multi-hop reasoning task at precision $p$ and batch size $B$. For any $p < \pi$ (native precision):
\begin{enumerate}
    \item \textit{Efficiency Recovery:} The hardware casting overhead is monotonically non-increasing with batch size, such that $\frac{\partial COR}{\partial B} \le 0$ and $\lim_{B \to \infty} \frac{\partial \mathcal{E}}{\partial B} < 0$.
    \item \textit{Trust Invariance:} The reasoning accuracy is invariant to batch size, such that $\frac{\partial T}{\partial B} = 0$.
\end{enumerate}
\label{thm:thm4}
\end{theorem} 

\begin{table*}[!htpb]
\centering
\small
\caption{Cross-backend validation (Mistral-7B, GSM8K, A100). AWQ resolves the efficiency trap (COR$<0$), but accuracy degradation persists, confirming Theorem~4.5.}
\label{tab:cross-backend}
\begin{tabular*}{\textwidth}{@{\extracolsep{\fill}} l c c c c c @{}}
\toprule
\textbf{Backend} & \textbf{Acc (\%)} & \textbf{TPS} & \textbf{E/Q (J)} & \textbf{COR} & \textbf{SI} \\
\midrule
FP16 Baseline     & \textbf{43.14} & 44.6          & 311.7          & 0.00  & 1.000 \\
AWQ 4-bit         & 40.71          & \textbf{45.9} & \textbf{184.7} & $-$0.03 & 0.981 \\
bnb-NF4 4-bit     & 39.42          & 36.5          & 503.1          & 0.22  & 0.845 \\
\bottomrule
\end{tabular*}
\end{table*}

\subsubsection{Cross-Backend Validation of Theorem 4.5}
\label{sec:cross-backend}

A key question is whether optimized quantization kernels 
can fully eliminate the Quantization Trap. To test this, 
we evaluate AWQ~\cite{lin2024awq} (fused GEMM kernels) 
alongside bitsandbytes (bnb) NF4 on Mistral-7B, using 
the full GSM8K dataset (1{,}319 samples) on an A100-80GB 
GPU. Here, we utilize NVML for energy monitoring, which 
differs from our primary Hardware-Latency Integration 
(HLI) approach. This alternative methodology serves to 
validate the Quantization Trap under different energy 
measurement conditions.

\begin{figure*}[!htpb]
\centering
\includegraphics[width=\textwidth]{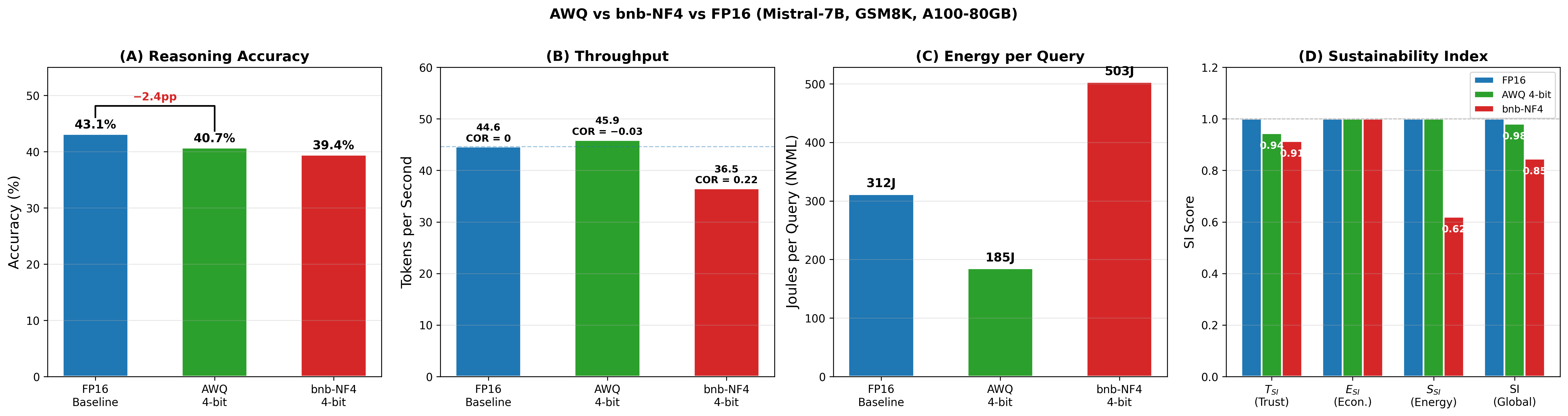}
\caption{Cross-backend validation of Theorem~4.5. (A)~Accuracy degrades under both 4-bit backends. (B)~AWQ achieves COR$<$0 (faster than FP16). (C)~AWQ uses 41\% less energy than FP16. (D)~SI decomposition: AWQ restores $E_{SI}$ and $S_{SI}$ to ceiling but $T_{SI}$ remains degraded.}
\label{fig:cross-backend}
\end{figure*}
As shown in Table~\ref{tab:cross-backend} and 
Figure~\ref{fig:cross-backend}, AWQ achieves 
COR~$=-0.03$ (faster than FP16) and consumes 41\% 
less energy, effectively eliminating the efficiency 
and energy components of the trap. However, reasoning 
accuracy remains 2.4~pp below the FP16 baseline 
(40.71\% vs.\ 43.14\%), with both 4-bit backends 
exhibiting similar deductive degradation (AWQ: 40.71\%, 
bnb-NF4: 39.42\%). 

This confirms that accuracy loss in multi-hop reasoning is an inherent property of the 
reduced bit-width, independent of kernel implementation, 
empirically validating Theorem~4.5 
($\partial T / \partial B = 0$). While optimized backends 
narrow the Quantization Trap operationally, \textit{the system 
remains trapped in the Trust dimension}, underscoring the 
diagnostic necessity of our multidimensional SI framework 
over single-scalar metrics.

\subsection{Per-Hop Stratification: Evidence of Sequential Deductive Fragility}
\label{sec:per-hop}

To empirically validate that quantization noise compounds sequentially, we stratify the GSM8K test set (1,319 samples) by the number of arithmetic steps required for the ground-truth solution. If \textit{Sequential Deductive Fragility} holds, accuracy degradation should scale disproportionately with reasoning depth.

\begin{table*}[t]
\centering
\small
\caption{Per-hop accuracy stratification (Mistral-7B, GSM8K). The FP16$\to$4-bit relative loss nearly doubles from 1--2 to 3--4 steps, providing direct evidence of cumulative quantization noise.}
\label{tab:per-hop}
\begin{tabular*}{\textwidth}{@{\extracolsep{\fill}} l r c c c c c @{}}
\toprule
\textbf{Depth} & \textbf{$n$} & \textbf{FP16 (\%)} & \textbf{8-bit (\%)} & \textbf{4-bit (\%)} & \textbf{Gap (pp)} & \textbf{Rel.\ Loss} \\
\midrule
1--2 steps & 427 & 57.4 & 56.0 & 53.4 & 4.0 & 6.9\% \\
3--4 steps & 663 & 43.0 & 39.7 & 38.0 & 5.0 & 11.6\% \\
5--6 steps & 199 & 18.6 & 20.6 & 20.1 & \multicolumn{2}{c}{\textit{floor effect}} \\
7+ steps   &  30 &  6.7 &  6.7 & 13.3 & \multicolumn{2}{c}{\textit{$n$ too small}} \\
\midrule
\textbf{Overall} & \textbf{1{,}319} & \textbf{43.1} & \textbf{41.3} & \textbf{39.7} & \textbf{3.4} & \textbf{7.9\%} \\
\bottomrule
\end{tabular*}
\end{table*}

\begin{figure*}[!htpb]
\centering
\includegraphics[width=\textwidth]{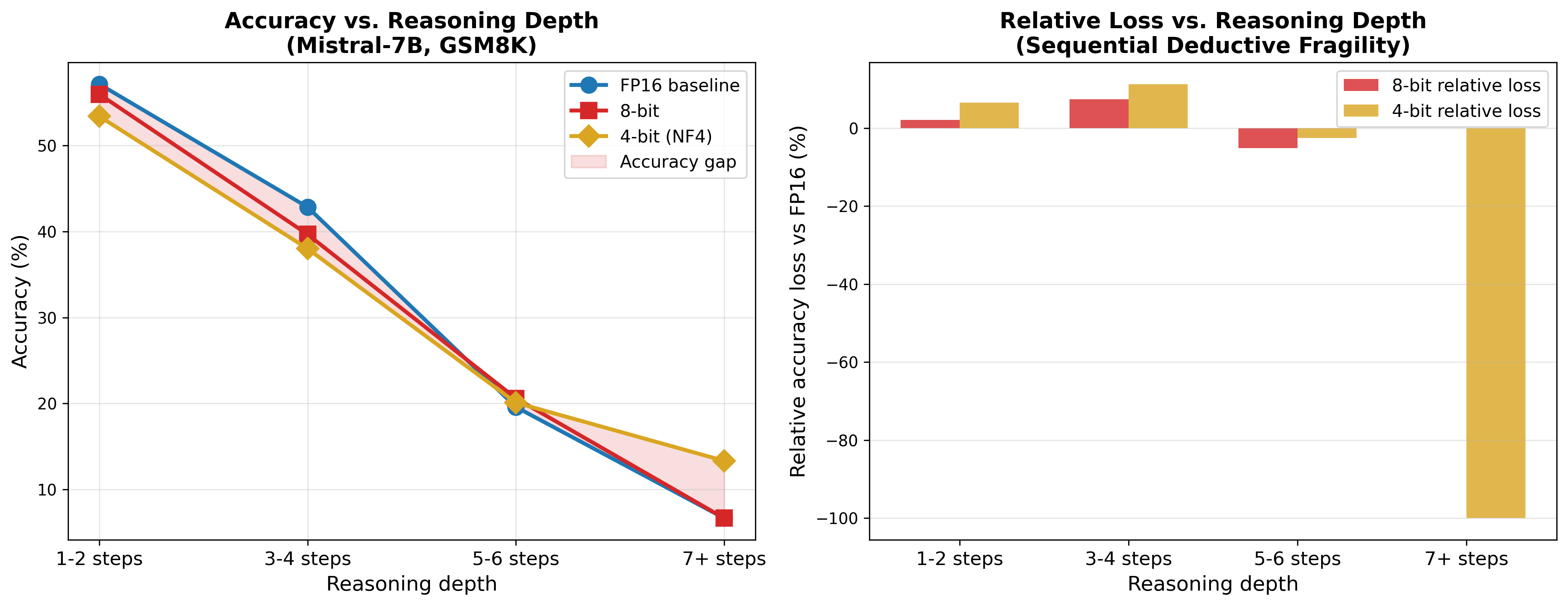}
\caption{Accuracy dynamics by reasoning depth. \textit{Left:} The absolute accuracy gap between FP16 and 4-bit widens as chain length increases. \textit{Right:} Relative loss nearly doubles from 1--2 to 3--4 steps, confirming that quantization errors compound multiplicatively across hops.}
\label{fig:per-hop}
\end{figure*}

\paragraph{Analysis.}
As shown in Table~\ref{tab:per-hop} and Figure~\ref{fig:per-hop}, the data reveals two distinct operational regimes:

\textit{1. The Compounding Regime (1--4 steps).} Within the model's capable reasoning depth (representing 82\% of the dataset), quantization noise visibly compounds. From 1--2 steps to 3--4 steps, the absolute FP16$\to$4-bit gap widens from 4.0 to 5.0 percentage points. More critically, the \textit{relative accuracy loss} nearly doubles from 6.9\% to 11.6\%. This confirms our theoretical mechanism: quantization is not a static per-query penalty, but a multiplicative per-hop cost ($\epsilon$) that fatally accumulates over extended logical chains.

\textit{2. The Entropy Floor (5+ steps).} Beyond 4 steps, the baseline FP16 accuracy collapses below 21\%, hitting the model's intrinsic cognitive ceiling. At this depth, the deductive signal vanishes and performance degrades to stochastic guessing. Consequently, the precision-induced gap is masked by general model failure, rendering low-bit vs. high-bit comparisons mathematically uninformative for this parameter scale.

Finally, the near-doubling of relative loss in the compounding regime provides direct empirical proof of \textit{Sequential Deductive Fragility}. It demonstrates that the Quantization Trap is not merely a byproduct of general task difficulty, but a structural inevitability driven by the sequential accumulation of sub-precision noise (Theorem 4.4).

\subsection{The Quantization Trap at Frontier Scale: 14B--32B Validation and Predictive Theory}
\label{sec:frontier-scale}

A critical question for the practical relevance of the Quantization Trap is whether it persists, deepens, or dissolves as models scale toward frontier deployment sizes. To answer this, we evaluate Qwen2.5-14B and Qwen2.5-32B under the identical 5-shot Chain-of-Thought (CoT) protocol across two hardware platforms (A100-SXM4-40GB and RTX PRO 6000 Blackwell) and GSM8K and MathQA. Our analysis now spans a $53\times$ scale ($0.6\text{B} \to 32\text{B}$) with NVML real-time energy monitoring.

\begin{table*}[t]
\centering
\small
\caption{Quantization results at 14B and 32B scales. The efficiency trap persists at 14B but exhibits batch-dependent dissolution at 32B, empirically validating the critical threshold $N^*$ (Theorem~\ref{thm:dis}).}
\label{tab:frontier}
\begin{tabular*}{\textwidth}{@{\extracolsep{\fill}} l l c c c c c @{}}
\toprule
\textbf{Benchmark} & \textbf{Mode} & \textbf{Acc (\%)} & \textbf{TPS} & \textbf{E/Q (J)} & \textbf{COR} & \textbf{VRAM (GB)} \\
\midrule
\multicolumn{7}{l}{\textit{Qwen2.5-14B --- MathQA, A100-SXM4-40GB}} \\
& 16-bit & 66.8 & \textbf{122.3} & \textbf{354} & 0.00 & 29.1 \\
& 8-bit  & \textbf{70.2} & 45.5 & 459 & 1.69 & 26.3 \\
& 4-bit  & 68.1 & 87.4 & 623 & 0.40 & 26.6 \\
\midrule
\multicolumn{7}{l}{\textit{Qwen2.5-14B --- GSM8K, RTX PRO 6000 Blackwell}} \\
& 16-bit & 76.9 & \textbf{189.4} & \textbf{280} & 0.00 & 31.5 \\
& 8-bit  & 75.0 & 80.7 & 372 & 1.35 & 28.5 \\
& 4-bit  & \textbf{80.4} & 134.8 & 473 & 0.41 & 28.9 \\
\midrule
\multicolumn{7}{l}{\textit{Qwen2.5-32B --- GSM8K, RTX PRO 6000 Blackwell (B=8$\to$16)}} \\
& 16-bit & 71.0 & 32.7 & 1517 & 0.00 & 66.2 \\
& 8-bit  & \textbf{73.3} & 31.3 & \textbf{858} & 0.05 & 36.5 \\
& 4-bit  & 72.9 & \textbf{61.1} & 1070 & $-$0.46 & \textbf{22.1} \\
\midrule
\multicolumn{7}{l}{\textit{Qwen2.5-32B --- MathQA, RTX PRO 6000 Blackwell (B=4$\to$8)}} \\
& 16-bit & 58.2 & \textbf{63.1} & 1055 & 0.00 & 66.8 \\
& 8-bit  & 51.8 & 59.2 & \textbf{642} & 0.07 & 37.5 \\
& 4-bit  & \textbf{64.5} & 36.5 & 2382 & 0.73 & \textbf{22.0} \\
\bottomrule
\end{tabular*}
\end{table*}

\subsubsection{Empirical Findings}

\textit{1. Trust Variance and Stochastic Regularization.}
Unlike the 7B scale where accuracy monotonically decays, the 14B and 32B models exhibit a striking reversal: 4-bit accuracy often \emph{exceeds} the FP16 baseline (e.g., 80.4\% vs.\ 76.9\% on GSM8K-14B). 

We formalize this by modeling multi-hop reasoning as a discrete dynamical system.  Let $W_{16}$ denote the native 16-bit weight matrix of the uncompressed model. Under greedy decoding (temperature $\tau=0$), execution through $W_{16}$ is strictly deterministic. Given a fixed 5-shot CoT prompt, the model traces a rigid logical trajectory that can often become trapped in deterministic failure loops or "hallucination snowballing" \cite{ho2020, zhang2023}, where early logical errors may irreversibly corrupt the sequence. 

We model the quantized weight matrix as $\hat{W}_p = W_{16} + \xi_p$, where $\xi_p \sim \mathcal{N}(0, \sigma^2_p)$ represents the quantization noise matrix. In overparameterized regimes ($\ge 14\text{B}$), $\xi_p$ acts as an implicit regularizer. The higher variance of 4-bit quantization ($\sigma^2_4 > \sigma^2_8$) occasionally induces \textit{Stochastic Resonance}, perturbing the weight matrix just enough to kick the generation out of deterministic failure loops. However, as shown below, this stochastic accuracy gain does not rescue the system from the physical scaling trap.

\textit{2. The Memory Asymptote (Amdahl's Law for VRAM).}
The standard economic justification for compression assumes memory scales linearly with bit-width ($M \propto p$). Yet, Table~\ref{tab:frontier} shows that compressing 14B from 16-bit to 4-bit yields an anemic 8.2\% memory savings (31.5\,GB to 28.9\,GB). 

This occurs because in $k$-shot multi-hop reasoning, total memory $M_{\text{total}}$ is the sum of static weights $M_w(p)$ and the dynamic Key-Value (KV) cache $M_{KV}(L)$, which depends strictly on context length $L$:
\begin{equation}
    M_{\text{total}}(p, L) = \beta \cdot p + M_{KV}(L)
\end{equation}
As $L$ grows during 5-shot CoT generation, $M_{KV}(L)$ aggressively dominates spatial complexity. Consequently, the partial derivative of total memory with respect to precision approaches zero:
\begin{equation}
    \lim_{L \to \infty} \frac{\partial M_{\text{total}}}{\partial p} \approx 0
\end{equation}
\textit{This mathematically guarantees that the ``smaller-is-better'' heuristic is economically void for long-context reasoning:} the footprint is bottlenecked by the unquantized KV-cache.

\textit{3. The Same Model, Two Opposite Outcomes.}
The most striking finding at 32B is that the \emph{same model} on the \emph{same GPU} exhibits opposite trap behaviors depending entirely on the workload:
\begin{itemize}
    \item \textit{GSM8K} (B=8$\to$16): The efficiency trap dissolves (COR$_{4\text{-bit}} = -0.46$). 4-bit is 87\% \emph{faster} than FP16 and uses 30\% less energy. 
    \item \textit{MathQA} (B=4$\to$8): The trap persists (COR$_{4\text{-bit}} = +0.73$). 4-bit is 42\% \emph{slower} than FP16 and consumes 2.3$\times$ more energy.
\end{itemize}

Since VRAM compression is identical in both cases 
(66$\to$22\,GB), the divergence is not driven by 
memory capacity. The mechanism is batch size: 
MathQA's structured prompts (3-shot CoT with 
five verbose answer options) produce longer input 
sequences, consuming more KV-cache memory per 
sample and constraining the maximum batch size 
to B=4 (FP16) and B=8 (4-bit), compared to 
B=8 and B=16 for GSM8K.

By Equation~\ref{eq:nst} in the following Theorem ~\ref{thm:dis}, 
halving the batch size doubles $N^*$, pushing 
the 32B model back below the dissolution threshold. 
This is a direct empirical confirmation that the 
Quantization Trap is not a fixed property of 
model scale, but a function of the deployment 
configuration size $N$, batch size $B$, 
hardware architecture $h$, and precision $p$.

\begin{theorem}[Critical Model Scale for Trap Dissolution]
\label{thm:dis}
Let $N$ denote the number of model parameters, $p$ and $\pi$ the quantized and native precisions, $\mathcal{B}_W$ the hardware memory bandwidth, $B$ the batch size, and $\phi(h,p)$ the per-forward-pass casting overhead. Define the Critical Model Scale as:
\begin{equation}
    N^* = \frac{\phi(h,p) \cdot \mathcal{B}_W \cdot B}{\alpha (\pi - p)}
    \label{eq:nst}
\end{equation}
where $\alpha$ is the memory access cost per bit. Then for all $N > N^*$, the Casting Overhead Ratio satisfies $\mathrm{COR} < 0$ (the efficiency trap dissolves), while for $N < N^*$, $\mathrm{COR} > 0$ (the trap is active).
\end{theorem} 

This Theorem (\textit{proof in Appendix)} makes three predictions, all confirmed by our data:
\begin{enumerate}
    \item \textit{$N \uparrow \implies$ trap dissolves:} $\mathrm{COR}$ drops from 0.41 (14B) to -0.46 (32B on GSM8K), locating $N^*$ strictly within this range.
    \item \textit{$B \uparrow \implies$ trap dissolves:} Equation~\ref{eq:nst} dictates $N^* \propto B$. Thus, the 32B model escapes the trap at larger batches (GSM8K, B=16) but falls back into it when context limits the batch size (MathQA, B=8). 
    \item \textit{$(\pi - p) \uparrow \implies$ trap dissolves:} 4-bit yields larger bandwidth savings than 8-bit, offsetting casting costs more effectively (on 32B GSM8K, COR$_{4\text{-bit}} = -0.46$ vs.\ COR$_{8\text{-bit}} = 0.05$).
\end{enumerate}

\subsubsection{Prediction for 70B+ Frontier Models}
Single-GPU FP16 inference at 70B requires $\sim$140\,GB VRAM, precluding controlled single-GPU experiments without introducing multi-GPU communication confounds. However, Theorem~\ref{thm:dis} yields a rigorous prediction:
\begin{corollary}[Residual Trap at Frontier Scale]
\label{cor:residual}
For models much larger than the critical scale 
($N \gg N^*$), Theorem~\ref{thm:dis} 
guarantees that the efficiency trap vanishes: 
quantized inference becomes faster and cheaper 
than FP16. However, by Theorem~\ref{thm:thm4} 
(Amortization-Trust Decoupling), reasoning 
accuracy is independent of throughput gains 
--- quantization noise still compounds across 
reasoning hops regardless of how efficiently the 
hardware executes each hop. 

Therefore, at frontier 
scale, the Quantization Trap does not disappear; 
it becomes \textbf{invisible}: the efficiency 
gains mask a latent accuracy deficit that only 
manifests in multi-hop reasoning.
\end{corollary}

This resolves an apparent contradiction in 
practitioner experience. At 70B+ scale, 4-bit 
models are genuinely faster, cheaper, and show 
negligible accuracy loss on single-turn benchmarks leading to the belief that quantization is 
``free.'' Our framework reveals that this belief 
is correct for efficiency ($E_{SI}$, $S_{SI}$) 
but wrong for multi-hop trust ($T_{SI}$): the 
per-hop noise compounding identified in 
Table~\ref{tab:per-hop} (6.9\%$\to$11.6\% 
relative loss) persists at any scale, and only 
becomes detectable when the reasoning chain is 
sufficiently long. The SI framework is designed 
precisely to expose this hidden cost.

The full COR trajectory from 0.6B to 72B, 
including the multi-GPU reversal, is synthesized 
in Figure~\ref{fig:nstar} and discussed in 
Section~\ref{sec:72b}.


\subsubsection{Unified Trap Taxonomy}
Synthesizing all experimental evidence from 0.6B to 32B, we identify three regimes of the Quantization Trap, predicted by the relationship between $N$ and $N^*(B)$:

\begin{enumerate}
    \item \textit{Full Trap} ($N \ll N^*$; 0.6B--7B): Accuracy degrades \emph{and} efficiency worsens. COR $>$ 1.0, energy spikes 2--4$\times$, and deductive trust decays. All three SI pillars favor FP16. 
    
    \item \textit{Efficiency-Only Trap} ($N \lesssim N^*$; 14B): Redundancy (Stochastic Regularization) absorbs quantization noise, preserving accuracy. However, unamortized casting overhead forces a latency-energy inversion. The penalty manifests purely as operational inefficiency.
    
    \item \textit{Conditional Dissolution} ($N \gtrsim N^*$; 32B): The trap dissolves \emph{only} if workloads permit large batch sizes (GSM8K, B=16). If sequence length constrains the batch size (MathQA, B=8), the trap re-emerges. 

\item \textit{Memory-Only Victory} ($N \gg N^*$; 72B): See Section~\ref{sec:72b} for controlled 
FP16 comparison on 3$\times$H100.


\end{enumerate}

\textit{No free quantization lunch:} In no regime does quantization deliver the \textit{unconditional equivalent quality at lower cost promised by linear scaling heuristics.} At small scales, it fails broadly; at large scales, it fails conditionally. The Quantization Trap is not a hardware glitch: it is a fundamental, conditional phase transition on the multi-hopping reasoning manifold.

\subsection{72B Frontier Validation: When Bigger Models Deepen the Trap}
\label{sec:72b}

We established that the Quantization Trap dissolves on a single GPU when model size exceeds $N^*$ (Theorem~\ref{thm:dis}). One might assume that at the 70B+ frontier scale, the trap vanishes entirely. Testing Qwen2.5-72B on both GSM8K and MathQA reasoning across three hardware configurations  reveals the exact opposite. Table~\ref{tab:72b} summarizes the results 
across two multi-GPU configurations.

\begin{table*}[t]
\centering
\small
\caption{Qwen2.5-72B across GSM8K and MathQA. 
On 3$\times$H100, accuracy is near-lossless on 
both benchmarks yet COR$=3.00$ on GSM8K: the 
\emph{highest} in our 120$\times$ study. On 
MathQA, 4-bit on a \emph{single} H100 achieves 
higher accuracy than FP16 on \emph{three} H100s, 
yet consumes 2.5$\times$ more energy (NVML).}
\label{tab:72b}
\begin{tabular*}{\textwidth}{@{\extracolsep{\fill}}
  llccccc@{}}
\toprule
\textbf{Benchmark (Hardware)} 
  & \textbf{Mode} 
  & \textbf{Acc (\%)} 
  & \textbf{TPS} 
  & \textbf{E/Q (J)} 
  & \textbf{COR} 
  & \textbf{VRAM (GB)} \\
\midrule
\multicolumn{7}{l}{\textit{GSM8K --- 
  3$\times$H100-80GB (B=8, TDP=2100W)}} \\
& 16-bit & \textbf{91.89} & \textbf{76.79} 
  & \textbf{11{,}804} & 0.00 & 151.7 \\
& 4-bit  & 91.66 & 19.19 
  & 47{,}592 & 3.00 & \textbf{63.2} \\
\midrule
\multicolumn{7}{l}{\textit{MathQA --- 
  3$\times$H100-80GB (FP16) vs.\ 
  1$\times$H100-80GB (4-bit)}} \\
& 16-bit (3$\times$) & 78.99 & \textbf{73.36} 
  & \textbf{5{,}750}$^\dagger$ & 0.00 & 153.3 \\
& 4-bit (1$\times$)  & \textbf{79.50} & 18.53 
  & 14{,}521$^\dagger$ & --- & \textbf{59.7} \\
\midrule
\multicolumn{7}{l}{\textit{GSM8K --- 
  2$\times$RTX 6000 Ada (B=4, TDP=1200W)}} \\
& 8-bit  & 76.72 & 8.59 
  & 20{,}943 & --- & 77.9 \\
& 4-bit  & \textbf{78.47} & \textbf{9.43} 
  & \textbf{19{,}088} & --- & \textbf{44.3} \\
\bottomrule
\multicolumn{7}{l}{\small $^\dagger$NVML 
  energy. All other E/Q values are TDP-based.}
\end{tabular*}
\end{table*}

\textit{1. The Accuracy Trap Dissolves.}
At 72B, accuracy is not merely preserved; it can improve. On GSM8K (3×H100), 4-bit trails FP16 by only 0.23pp (91.66\% vs 91.89\%)—the smallest gap in our 120× campaign. On MathQA, 4-bit on a single H100 actually exceeds the FP16 baseline run on three H100s (79.50\% vs 78.99\%, +0.51pp). On 2$\times$RTX 6000 Ada, 4-bit even outperforms 8-bit (78.47\% vs.\ 76.72\%). The stochastic regularization effect, first observed at 14B, is now confirmed at frontier scale across two benchmarks and three hardware configurations. By every accuracy metric, 4-bit quantization at 72B is "lossless." A practitioner evaluating accuracy alone would deploy with complete confidence. This is precisely what makes the following findings dangerous.

\textit{2. The bigger model, the more severe 'Trap Explodes'.}
Here lies the paradox. Despite near-perfect accuracy, the 3$\times$H100 setup reveals \emph{the most severe efficiency trap in our study}: COR$=3.00$. The 4-bit model is 4$\times$ slower than FP16 (19.19 vs.\ 76.79 TPS) and consumes 4$\times$ more energy (47{,}592 vs.\ 11{,}804\,J per query). 

Strikingly, this COR$=3.00$ at 72B exceeds the COR$=1.80$ at 0.6B. \emph{Our largest model is more deeply trapped than our smallest}, shattering the assumption that scaling inherently resolves the Quantization Trap.

\textit{MathQA speaks louder:} The MathQA results shatter the "linear scaling" assumption: less hardware does not equal less energy: it goes the opposite.  When pitting 4-bit quantization on a single H100 (700W) against full-precision FP16 across three H100s (2,100W), the quantized model actually consumes more energy per query. Despite using one-third the hardware, the 4-bit setup incurs a massive 2.5$\times$ energy penalty in actual usage (14,521 J vs. 5,750 J) and a 1.3$\times$ penalty under theoretical max limits (TDP). 

As proven in Theorem~\ref{thm:Bstar} (\textit{Sequential Amortization Failure}), the unamortized "casting overhead" of dequantizing weights at every reasoning step can completely eclipse all memory bandwidth savings.

The takeaway is stark: quantization actively wastes more energy than full precision, even when drastically reducing your GPU footprint while preserving accuracy.

\textit{3. Why Scale Deepens the Trap on Multi-GPU?}
The dramatic reversal from COR$=-0.46$ (32B, single-GPU) to COR$=+3.00$ (72B, multi-GPU) stems from hardware saturation dynamics:

\begin{enumerate}
    \item \textit{Single-GPU (32B):} The FP16 model consumes 66.2\,GB of a 96\,GB GPU. Inference is deeply \emph{memory-bandwidth-bound}, so quantization's 3$\times$ weight compression yields massive bandwidth savings that offset casting overheads. Result: COR$<0$.
    
    \item \textit{Multi-GPU (72B):} With 240\,GB of aggregate VRAM (3$\times$80\,GB), the 151.7\,GB FP16 model fits with 88\,GB to spare. FP16 is no longer memory-bound, instantly evaporating quantization's bandwidth advantage. Consequently, casting overheads dominate, compounded by inter-GPU synchronization delays.
\end{enumerate}

This aligns with Theorem~\ref{thm:dis}. The threshold formula $N^* = \phi \cdot \mathcal{B}_W \cdot B / \alpha(\pi - p)$ implicitly assumes a \emph{fixed} hardware bandwidth $\mathcal{B}_W$. Multi-GPU setups multiply effective bandwidth, raising $N^*$ proportionally. A 72B model that exceeds $N^*$ on a hypothetical 200\,GB single GPU falls \emph{below} $N^*$ on a 3$\times$80\,GB cluster where FP16 runs unconstrained.

\textit{4. The Memory-Only Victory.}
As the quantization scorecard (Table~\ref{tab:72b-scorecard}) reveals, at 72B quantization delivers on exactly two of its four promises: accuracy is preserved (even improving slightly on MathQA) and memory is reduced by roughly 60\%. 

The 58--61\% VRAM reduction provides a massive economic benefit. As explicitly demonstrated by the MathQA configuration, it enables 72B deployment on a single 80\,GB GPU instead of three. However, this hardware consolidation comes at a severe operational cost: throughput universally degrades by 4$\times$ across both benchmarks, and energy consumption surges by 2.5$\times$ to 4.0$\times$ per query, even when using fewer GPUs. 

In SI terms, the divergence is extreme: $T_{SI} \approx 1.00$, but $E_{SI} \approx 0.25$ and $S_{SI}$ drops to between 0.25 and 0.40. Accuracy-only metrics would declare this ``lossless,'' while throughput-only metrics would deem it ``catastrophic.'' The SI framework uniquely captures both realities. This reflects the true dilemma for practitioners serving frontier models: they choose 4-bit not because it performs better, but because they \emph{cannot afford} the FP16 hardware. At the frontier scale, the Quantization Trap is an economic necessity.

\begin{table*}[!htpb]
\centering
\small
\caption{The 72B quantization scorecard. On both 
benchmarks, quantization preserves accuracy and 
reduces memory but severely degrades throughput 
and energy even when using fewer GPUs.}
\label{tab:72b-scorecard}
\begin{tabular*}{\textwidth}{@{\extracolsep{\fill}}lcccc@{}}
\toprule
\textbf{Promise} & \textbf{Delivered?} 
  & \textbf{FP16} & \textbf{4-bit} 
  & \textbf{Penalty} \\
\midrule
\multicolumn{5}{l}{\textit{GSM8K 
  (3$\times$H100 vs.\ 3$\times$H100)}} \\
Accuracy\ preserved 
  & \checkmark & 91.89\% & 91.66\% 
  & $-$0.23\,pp \\
Faster inference 
  & $\times$ & 76.79 TPS & 19.19 TPS 
  & 4.0$\times$ slower \\
Lower energy 
  & $\times$ & 11{,}804\,J & 47{,}592\,J 
  & 4.0$\times$ more \\
Less memory 
  & \checkmark & 151.7\,GB & 63.2\,GB 
  & 58\% saved \\
\midrule
\multicolumn{5}{l}{\textit{MathQA 
  (3$\times$H100 FP16 vs.\ 
  1$\times$H100 4-bit)}} \\
Accuracy\ preserved 
  & \checkmark & 78.99\% & \textbf{79.50\%} 
  & $+$0.51\,pp \\
Faster inference 
  & $\times$ & 73.36 TPS & 18.53 TPS 
  & 4.0$\times$ slower \\
Lower energy$^\dagger$ 
  & $\times$ & 5{,}750\,J & 14{,}521\,J 
  & 2.5$\times$ more \\
Less memory 
  & \checkmark & 153.3\,GB & 59.7\,GB 
  & 61\% saved \\
\bottomrule
\multicolumn{5}{l}{\small $^\dagger$NVML 
  energy. GSM8K values are TDP-based.}
\end{tabular*}
\end{table*}

\begin{figure}[!htpb]
\centering
\includegraphics[width=\columnwidth]{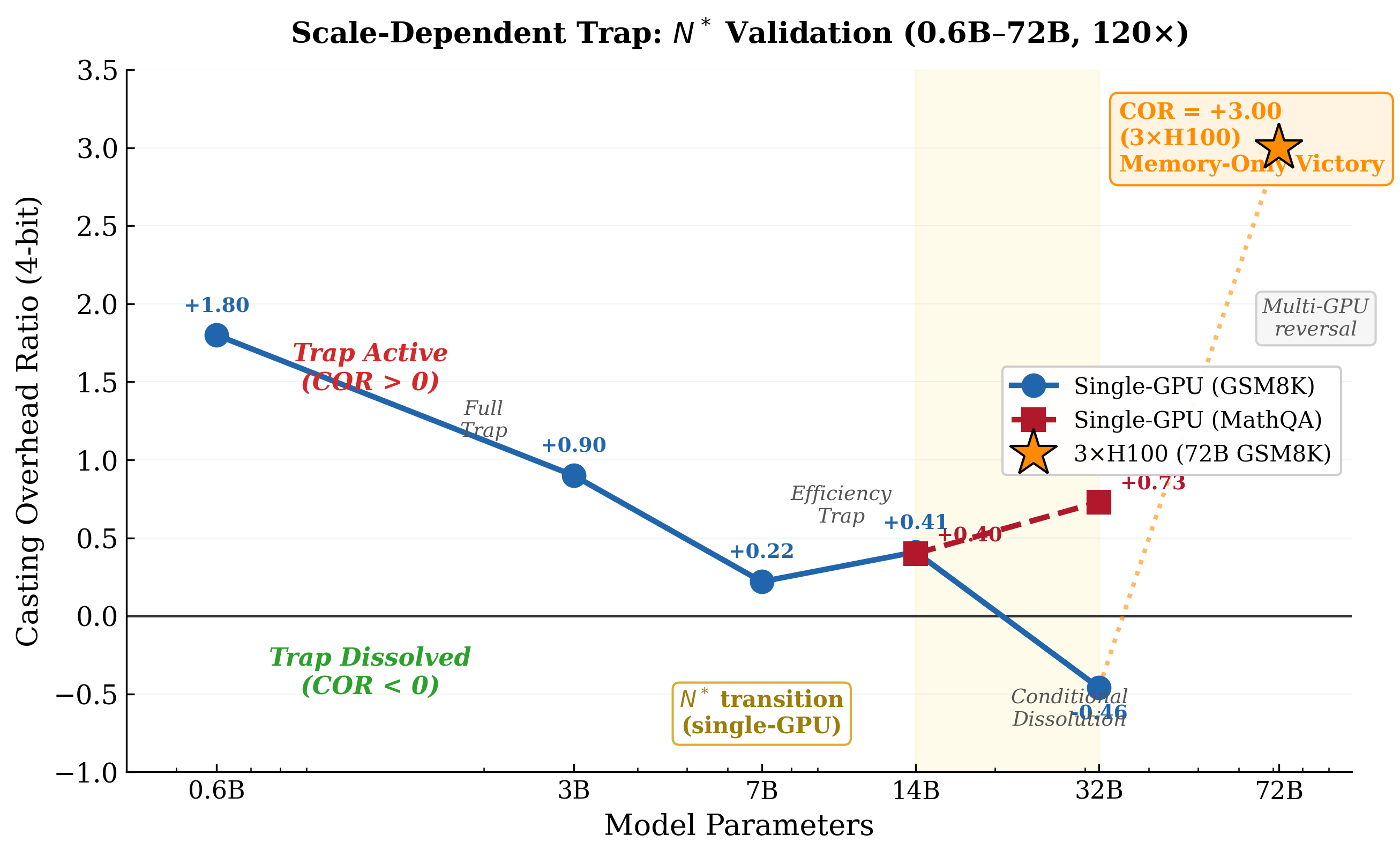}
\caption{COR trajectory from 0.6B to 72B (120$\times$). 
The single-GPU line (blue) crosses zero between 14B 
and 32B, locating $N^*$. MathQA (red) remains trapped 
at 32B due to smaller batches. The 72B point (orange 
star, 3$\times$H100) reveals a dramatic reversal: 
COR$=+3.00$ because FP16 fits with headroom, 
removing the bandwidth bottleneck. Trap dissolution 
is a property of hardware saturation, not model 
scale alone (Theorem~\ref{thm:dis})}.
\label{fig:nstar}
\end{figure}

 \textit{5. The Complete 120$\times$ Trajectory.} Figure~\ref{fig:nstar} maps COR from 0.6B to 72B. On a single GPU, GSM8K (blue) drops from +1.80 at 0.6B to -0.46 at 32B, crossing zero to establish the $N^*$ threshold. MathQA (red) stays positive at 32B (+0.73), proving batch size alters this boundary. Crucially, scaling to 72B on 3$\times$H100 (orange star) triggers a massive reversal to COR = +3.00. With ample multi-GPU VRAM for FP16, quantization's bandwidth advantage evaporates and casting overhead dictates performance. The yellow transition zone confirms $N^*$ is highly hardware-dependent, not universal.

 The four regimes: Full Trap (0.6B--7B), 
Efficiency-Only (14B), Conditional Dissolution 
(32B), and Memory-Only Victory (72B) are 
annotated in Figure~\ref{fig:nstar} and 
constitute the complete trap taxonomy validated 
across 120$\times$ of model scale, and confirmed on both GSM8K and MathQA.

This non-monotonic trajectory carries a sobering implication: \emph{the industry trend of serving larger models on multi-GPU clusters does not escape the Quantization Trap, it re-enters it from a different direction.} While single-GPU deployments of 32B models hit a sweet spot where quantization genuinely improves efficiency, scaling to frontier models on multi-GPU clusters provides enough hardware abundance to run FP16 efficiently. This eliminates quantization's bandwidth advantage and reinstates the full casting overhead penalty: completing the trap 
cycle.

\section{Discussion and Conclusion}

In this work, we introduced the Sustainability Index (SI) framework to demonstrate that the linear ``smaller-is-more-efficient'' scaling heuristic breaks down in multi-hop reasoning. While prior literature highlights quantization's energy benefits for single-turn tasks and standard inference workloads \cite{Ra2024,Hu2025}, we show that sequential logical chains fundamentally alter these dynamics. The compounding accumulation of hardware-level casting overhead and quantization noise across successive reasoning steps drives the system into a \emph{Quantization Trap}. Consequently, while aggressive quantization optimizes energy for shallow tasks, it paradoxically increases net energy consumption and degrades deductive accuracy in complex reasoning environments. We plan to address this gap in future work by framing the allocation of numerical precision as a game-theoretic equilibrium between hardware energy constraints and deductive trust.

\textit{Next-Generation Hardware and Deductive Fragility.} 
Emerging architectures (e.g., NVIDIA's Blackwell \cite{nvidia2024b}) and optimized kernels (e.g., AWQ) offer a path to neutralize the \emph{efficiency} component of the trap. By providing native support for low-bit formats and fused operations, they drive casting overhead toward zero ($\phi \to 0$). Our cross-backend experiments (Section~\ref{sec:cross-backend}) validate this: AWQ achieves a casting overhead ratio (COR) of $-0.03$ on an A100, fully eliminating the energy and throughput deficits. 

However, hardware acceleration cannot resolve the \emph{accuracy} trap. Quantization noise compounds across reasoning hops, degrading deductive trust ($T_{SI}$) in a manner that cannot be amortized by batching or kernel speedups. Table~\ref{tab:per-hop} demonstrates this multi-hop fragility: the relative accuracy loss nearly doubles from 6.9\% at 1--2 steps to 11.6\% at 3--4 steps. Because reasoning remains structurally sequential ($B \approx 1$), scaling-law divergence persists at low bit-widths (Figure~\ref{fig:fig4}C) regardless of hardware efficiency. 

\textit{Scale, Hardware, and the 72B Reversal.} 
Our analysis further reveals that the Quantization Trap is non-monotonic with model scale. At 32B, a single GPU faces severe memory pressure; here, quantization's bandwidth compression genuinely improves efficiency (COR = $-0.46$ on GSM8K). Yet this dissolution is highly conditional: the same 32B model re-enters the trap on MathQA (COR = $+0.73$) when longer sequences force smaller batch sizes. 

Most strikingly, our frontier experiments (Section~\ref{sec:72b}) demonstrate that multi-GPU clusters \emph{deepen} the trap. Serving Qwen2.5-72B on 3$\times$H100 provides near-lossless accuracy ($-0.23$\,pp), but yields a COR of $+3.00$—the highest in our entire $120\times$ study. When aggregate multi-GPU VRAM provides sufficient headroom to run FP16 without memory bottlenecks, quantization's bandwidth advantage instantly evaporates. Casting and communication overheads reassert dominance. Thus, the industry trend of serving larger models on multi-GPU infrastructure does not escape the Quantization Trap; it merely re-enters it from a hardware-abundant direction.

Ultimately, even as hardware evolves and models scale, the simple linear heuristic ($E \propto \text{bits}$) is obsolete. Precision-efficiency trade-offs remain rigorously task-, scale-, and hardware-conditional.

\textit{The Fundamental Lesson.} Across a 120$\times$ model scale (0.6B--72B), six GPU architectures, and three backends, one invariant emerges: \emph{there is no free quantization lunch for multi-hop reasoning.} Quantization invariably exacts a toll—sacrificing accuracy at small scales, efficiency at mid scales, and energy at the multi-GPU frontier. The penalty changes shape, but never vanishes. Ultimately, the Quantization Trap is not a transient bug to be patched by faster kernels; it is a structural property of sequential deduction. It can only be exposed through multidimensional evaluation and navigated by practitioners who recognize their specific hardware-scale regime.

\textit{Limitations.} 
Our analysis focuses on \emph{weight-only} quantization (W4A16, W8A16); weight-activation quantization targets a distinct, compute-bound regime requiring separate investigation. Our empirical range spans a 120$\times$ scale (0.6B--72B) across six GPU architectures (L4, A100, H100, RTX PRO 6000 Blackwell, RTX 6000 Ada) with cross-backend validation. In our 72B multi-GPU results, inter-GPU communication acts as an additional confound; isolating this latency from pure casting overhead is an important direction for future work. Finally, while we use \texttt{bitsandbytes} as the most widely deployed baseline, our SI framework is method-agnostic and can diagnose any quantization backend.

\textbf{Data and code:} \textit{Public benchmarks used; software available upon acceptance}

\section*{Impact Statement}
This research exposes the ”Quantization Trap” and characterizes a fundamental phenomenon of \textit{scaling law breaking} in multi-hop reasoning. We provide a rigorous framework to prevent ”Sustainability Inversion,” where aggressive model compression paradoxically increases energy consumption and carbon emissions while degrading deductive
reliability. By mathematically demonstrating where and
why linear scaling laws fail, our findings challenge the
industry-standard heuristic that smaller models are inherently more efficient. This work encourages a transition
toward precision-aware scaling and hardware, software co-design, ensuring that AI advancement remains environmentally responsible and trustworthy in safety-critical domains
such as high-stakes medical and legal reasoning where logical fidelity must not be sacrificed for misleading gains in
memory footprint.

\textit{Quantization Trap is a sequential reasoning trap:} The Quantization Trap is not confined to arithmetic; it infects any task requiring fragile, multi-step logic. In clinical medicine, a single quantization-induced error early in a diagnostic chain can cascade, producing a highly confident treatment plan built on corrupted premises. Legal interpretation, complex code generation, and autonomous agent workflows face the exact same vulnerability. Crucially, Our data shows error rates compound exponentially after four logical steps, failing in the exact high-stakes scenarios where AI is most critical. The Quantization Trap is fundamentally a \textit{sequential reasoning trap} inherited by any model that thinks in chains.

\appendix
\onecolumn
\section{Appendix}

\begin{proposition}[Scaling Law Divergence]
Let $SI(\theta)$ be a global utility function of precision $p$. A \textit{Quantization Trap} is identified when: 
\begin{equation}
    \frac{\partial SI}{\partial p} > 0 
\end{equation}
signifying a fundamental breakdown of the linear scaling law.
\label{prop:slb}
\end{proposition}

\begin{proof}
By definition, $SI(p) = \sum_{i \in \{T,E,S\}} w_i v_i(p)$ where $w_i > 0$. In a standard scaling regime, utility increases as precision decreases ($\frac{\partial SI}{\partial p} < 0$). If the constituent sustainability pillars $v_i$ (Trust, Efficiency, Energy) instead exhibit positive gradients such that their weighted sum is positive, the system enters a state of divergence. In this state, the system is strictly dominated by its high-precision counterpart, rendering bit-reduction counterproductive.
\end{proof}

\begin{lemma}[Average Latency per hop]
Let $TPS$ be the tokens per second. The average latency per hop is $\tau = 1/TPS$.
\end{lemma}

\begin{proof}
Let $K$ be the total number of atomic hops (equivalent to the total tokens generated) produced during a reasoning query of total duration $T$. By definition, the throughput $TPS$ is the rate of tokens produced per second:
$TPS = \frac{K}{T}$
    
The average latency per hop $\tau$ is the temporal cost incurred per atomic forward pass: $\tau = \frac{T}{K}.$ It follows by substitution that:
$\tau = \frac{1}{K/T} = \frac{1}{TPS}$

This identity establishes that the per-hop temporal cost is the mathematical reciprocal of the system's token production rate.
\end{proof}

\begin{proposition}[$COR$ Approximation]
Given a high-precision reference model (ref) and a quantized model ($p$), the Casting Overhead Ratio is given by:
\begin{equation}
    COR \approx \frac{TPS_{\text{ref}}}{TPS_p} - 1
\end{equation}
\end{proposition}

\begin{proof}
We decompose the average latency per hop of the quantized model, $\tau_p$, into its constituent hardware components: the invariant native compute time ($\tau_{\text{comp}}$) and the de-quantization overhead ($\tau_{\text{cast}}$):
\begin{equation}
    \tau_p = \tau_{\text{comp}} + \tau_{\text{cast}}
\end{equation}
For the high-precision reference model, the casting overhead is negligible ($\tau_{\text{cast}} \approx 0$), implying $\tau_{\text{ref}} = \tau_{\text{comp}}$. Substituting this identity into the formal definition of the Casting Overhead Ratio ($COR = \tau_{\text{cast}}/\tau_{\text{comp}}$) yields:
\begin{equation}
    COR = \frac{\tau_p - \tau_{\text{ref}}}{\tau_{\text{ref}}} = \frac{\tau_p}{\tau_{\text{ref}}} - 1
\end{equation}
Invoking the \textit{Throughput-Latency Identity} ($\tau = 1/TPS$), we transform the per-hop temporal costs into observable production rates:
\begin{equation}
    COR = \frac{1/TPS_p}{1/TPS_{\text{ref}}} - 1
\end{equation}
Simplifying the complex fraction, we arrive at the final approximation:
\begin{equation}
    COR \approx \frac{TPS_{\text{ref}}}{TPS_p} - 1
\end{equation}
This identity grounds the hardware "tax" in macro-level telemetry, proving that $COR$ represents the normalized throughput deficit relative to the baseline.
\end{proof}

\begin{theorem}[Sequential Amortization Failure]
Let $\mathcal{E}(p, B)$ be the energy-per-query as a function of precision $p$ and batch size $B$. For a multi-hop reasoning task on hardware with native precision $\pi > p$, there exists a \textit{Critical Batch Threshold} $B^*$ such that for all $B < B^*$, the system enters a \textit{ a Quantization Trap} defined by:
\begin{equation}
    \frac{\partial \mathcal{E}}{\partial p} < 0
\end{equation}
This inversion occurs when the marginal casting overhead $\phi$ outpaces the amortized bandwidth savings, where $B^*$ is uniquely determined by:
\begin{equation}
    B^* = \alpha \frac{(\pi - p)}{\phi(h, p)}
\end{equation}
Here, $\alpha$ is the hardware energy constant for memory movement and $\phi(h, p)$ is the de-quantization cost for precision $p$ on architecture $h$.
\label{thm:saf}
\end{theorem}

\begin{proof}
\textit{1. Energy Functional Decomposition:} 
We define total energy $\mathcal{E}(p, B)$ for a reasoning task of depth $K$. On hardware $h$ with native precision $\pi$, energy per query is:
\begin{equation}
    \mathcal{E}(p, B) = K \left[ \gamma + \frac{\alpha p}{B} + \phi(h, p) \right]
\end{equation}
where $\gamma$ is the static compute energy per hop, $\alpha p/B$ is the weight-movement energy, and $\phi$ is the casting overhead.

The equation calculates the total energy as the product of the reasoning depth ($K$) and the sum of three physical energy components required for every step in the chain.

\begin{itemize}
 \item static compute energy per hop ($\gamma$): the constant power consumed by the GPU’s arithmetic units to execute logic gates.

\item Data Movement Energy ($\frac{\alpha p}{B}$) per hop: The energy required to transport model weights from the memory (VRAM) to the processing cores. This is determined by the bit-count ($p$), batch size ($B$) and hardware energy constant for memory movement  $\alpha$.

\item Casting Overhead per hop ($\phi$): overhead from casting low-bit width to the default one.
\end{itemize}

The equation defines total query energy as the product of reasoning depth ($K$) and the combined per-step costs of static hardware execution, the batch-shared energy used to move weights from memory, and the specific overhead required to translate compressed data into a format the GPU can process natively.

\textit{2. Existence of the critical batch threshold $B^*$:} 

 In the native scenario ($\pi$), there is zero conversion overhead because the hardware processes the data natively ($\phi=0$). In the Low-Precision scenario ($p$), the movement energy is lower, but we must add the conversion overhead ($\phi$). 
    We look for the batch size $B^*$ where:
    $$\text{Energy}_{\text{Low-Bit}} = \text{Energy}_{\text{Native}}$$

Mathematically, it means 
\begin{equation*}
    K \left[ \gamma + \frac{\alpha p}{B^*} + \phi(h, p) \right] = K \left[ \gamma + \frac{\alpha \pi}{B^*} + 0 \right]
\end{equation*}
Subtracting common terms $\gamma$ and dividing by $K$:
\begin{equation*}
    \frac{\alpha p}{B^*} + \phi(h, p) = \frac{\alpha \pi}{B^*} \implies \phi(h, p) = \frac{\alpha(\pi - p)}{B^*} \implies B^* = \alpha \frac{(\pi - p)}{\phi(h, p)}
\end{equation*}

When $B<B^*$, low-bit model will spend more energy than the original model by spending more energy.

\textit{3.  Quantization Trap Gradient:} 
The occurs when $\frac{\partial \mathcal{E}}{\partial p} < 0$ (see Prop.\ref{prop:slb}.) Differentiating the energy functional:
\begin{equation}
    \frac{\partial \mathcal{E}}{\partial p} = K \left[ \frac{\alpha}{B} + \frac{\partial \phi}{\partial p} \right]
\end{equation}
In sequential reasoning, the de-quantization function $\phi(h, p)$ is a step-wise non-increasing function: \textit{larger p,  less conversion overhead:} ($\partial \phi / \partial p \leq 0$). 

At low batch sizes ($B < B^*$), the negative magnitude of the casting penalty $|\partial \phi / \partial p|$ exceeds the marginal bandwidth savings $\alpha/B$, causing the gradient to sign-flip. This confirms that increasing precision $p$ paradoxically reduces energy. $\blacksquare$
\end{proof}

\begin{theorem}[Scaling Law Divergence in Multi-Hop Reasoning]
For multi-hop reasoning tasks on hardware with native precision $\pi > p$, the Scaling Law Divergence ($\frac{\partial SI}{\partial p} > 0$) is structurally unavoidable for batch sizes $B < B^*$.
\label{thm:reasoning_divergence}
\end{theorem}

\begin{proof}
We evaluate the three pillars of Proposition \ref{prop:slb} under the conditions of multi-hop reasoning:

\begin{enumerate}
    \item \textit{Trust ($T_{SI}$):} Because reasoning is a sequential product of logical transitions $P(y) = \prod P(h_k)$, quantization noise $\epsilon$ compounds at each hop. Thus, accuracy is a strictly non-decreasing function of precision, $\frac{\partial T_{SI}}{\partial p} > 0$.
    
    \item \textit{Energy ($S_{SI}$):} From Theorem \ref{thm:saf}, when $B < B^*$, the energy-per-query $\mathcal{E}$ decreases as precision $p$ increases because the hardware-level casting tax $\phi(h, p)$ is eliminated. Since $S_{SI}$ is inversely related to energy consumption, $\frac{\partial S_{SI}}{\partial p} > 0$.

    \item \textit{Economy ($E_{SI}$):} In the sequential regime ($B \approx 1$), the \textit{Casting Overhead Ratio} ($COR$) dominates the execution pipeline. Restoring precision removes the software-emulated de-quantization bottleneck, increasing throughput such that $\frac{\partial E_{SI}}{\partial p} > 0$.
\end{enumerate}

Since all partial derivatives are positive, the global gradient $\frac{\partial SI}{\partial p} = \sum w_i \frac{\partial v_i}{\partial p} > 0$. By Proposition \ref{prop:slb}, the system is trapped in a scaling law divergence.
\end{proof}

\begin{theorem}[Amortization-Trust Decoupling]
\label{thm:decoupling}
Let $\mathcal{E}(p, B)$ denote the energy-per-query and $T(p, B)$ denote the reasoning accuracy for a multi-hop task at precision $p < \pi$. For any increasing batch size $B$:
\begin{enumerate}
\item \textit{Efficiency Recovery:} The hardware casting overhead is monotonically non-increasing, $\frac{\partial \mathrm{COR}}{\partial B} \le 0$, and energy-per-query decreases, $\frac{\partial \mathcal{E}}{\partial B} < 0$.
\item \textit{Trust Invariance:} The reasoning accuracy is invariant to batch size, $\frac{\partial T}{\partial B} = 0$.
\end{enumerate}
\end{theorem}

\begin{proof}
\textit{1. Proving Efficiency Recovery:}
In current hardware architectures, model weights are loaded and de-quantized once per batch per token. We model the per-hop latency components as:

\begin{equation}
    \tau_{\mathrm{comp}}(p, B) = a_{\mathrm{comp}}(p) \cdot B, \quad \tau_{\mathrm{cast}}(p, B) = a_{\mathrm{cast}}(p)
\end{equation}

where $a_{\mathrm{comp}}$ is the compute cost per example and $a_{\mathrm{cast}}$ is the fixed de-quantization overhead.
Substituting these into the definition of COR:
\begin{equation}
    \mathrm{COR}(p, B) = \frac{a_{\mathrm{cast}}(p)}{a_{\mathrm{comp}}(p) \cdot B} \implies \frac{\partial \mathrm{COR}}{\partial B} = -\frac{a_{\mathrm{cast}}(p)}{a_{\mathrm{comp}}(p) \cdot B^2} \le 0
\end{equation}

This proves that the 'translation tax' is a fixed cost shared by everyone in the batch. As you add more queries ($B$), the share of the tax paid by each query shrinks, allowing hardware efficiency to recover.

Next, we define Energy-per-Query as a function of the shared casting energy ($e_{\mathrm{cast}}$) and the compute energy ($e_{\mathrm{comp}}$):
\begin{equation}
    \mathcal{E}(p, B) = K \left( e_{\mathrm{comp}}(p) + \frac{e_{\mathrm{cast}}(p)}{B} \right) \implies \frac{\partial \mathcal{E}}{\partial B} = -\frac{K \cdot e_{\mathrm{cast}}(p)}{B^2} < 0
\end{equation}

Because the total energy to move and de-quantize weights is constant, processing more queries simultaneously reduces the energy 'bill' for each individual query.

\textit{2. Proving Trust Invariance:}
Reasoning accuracy $T(p, B)$ is determined by the model's logits, which are a function of the precision $p$. For a fixed decoding algorithm, the output $\hat{y}$ for an input $x$ is:
\begin{equation}
    \hat{y} = f(x, p)
\end{equation}
Since $f(x, p)$ has no dependency on the batch size $B$ (Assumption of Semantic Independence), the prediction for any query in the batch remains identical regardless of parallelization:
\begin{equation}
    T(p, B) = \frac{1}{N} \sum_{i=1}^{N} \mathbb{I}(\hat{y}_{p,i} = y^*_i) \implies \frac{\partial T}{\partial B} = 0
\end{equation}

This proves the decoupling. Batching makes the hardware faster, but it does not make the model smarter. The reasoning errors caused by 4-bit quantization are a permanent property of the bits ($p$), not the batch ($B$). Mitigation via batching fixes the energy 'Trap' but leaves the \textit{'Deductive Fragility'} untouched.
\end{proof}

\begin{theorem}[Critical Model Scale for Trap Dissolution]
\label{thm:dissolution}
Let the efficiency of autoregressive multi-hop reasoning be governed by the interplay between memory bandwidth and computational overhead. We define the following system parameters:
\begin{itemize}
    \item $N \in \mathbb{N}$: The total parameter count of the model (governing the memory payload).
    \item $\pi, p \in \mathbb{N}$: The native and quantized bit-widths, respectively (e.g., $\pi=16, p=4$).
    \item $\mathcal{B}_W$: The theoretical peak memory bandwidth of the hardware (in bits per second).
    \item $B$: The batch size (number of concurrent logical chains processed).
    \item $\phi(h, p)$: The unamortized hardware casting latency (in seconds) per atomic forward pass, representing the fixed kernel overhead required to de-quantize $p$-bit weights into native registers on hardware $h$.
    \item $\alpha \ge 1$: A hardware-specific memory access constant accounting for effective bandwidth utilization (routing latency, cache misses).
\end{itemize}

We define the \textit{Critical Model Scale} $N^*$ as the exact parameter boundary where bandwidth savings perfectly neutralize the de-quantization penalty:
\begin{equation}
    N^* = \frac{\phi(h,p) \cdot \mathcal{B}_W \cdot B}{\alpha (\pi - p)}
    \label{eq:nstar}
\end{equation}
Then, for all models where $N > N^*$, the Casting Overhead Ratio satisfies $\mathrm{COR} < 0$ (the efficiency trap dissolves). Conversely, for $N < N^*$, $\mathrm{COR} > 0$ (the efficiency trap is active).
\end{theorem}

\begin{proof}
\textit{1: Formalize Latency in the Memory-Bound Regime.}\\
During the decoding phase of multi-hop reasoning (where $B$ is strictly constrained by context length), execution is heavily memory-bandwidth bound. The actual matrix-vector multiplication (compute) is completely masked by the time required to fetch weights from VRAM. Let $\tau$ denote the average latency per token generated (i.e., $\tau = 1/\text{TPS}$). 

For the $\pi$-bit native model, weights are loaded directly without translation. The total time to fetch $N$ parameters of size $\pi$ bits is $(\alpha N \pi) / \mathcal{B}_W$. Since $B$ queries share this single weight fetch during a forward pass, the amortized latency per query is:
\begin{equation}
    \tau_{\mathrm{ref}} = \frac{\alpha N \pi}{\mathcal{B}_W \cdot B}
\end{equation}

\textit{2: Isolate the Quantized Latency Components.}\\
For the $p$-bit quantized model, the latency decomposes into two sequential physical events: (i) fetching the compressed weights, and (ii) executing the software-emulated de-quantization kernels. Thus, the per-query latency is:
\begin{equation}
    \tau_q = \frac{\alpha N p}{\mathcal{B}_W \cdot B} + \phi(h, p)
\end{equation}
where $\phi(h,p)$ acts as an additive hardware tax that scales independently of the memory bus.

\textit{3: Resolve the Trap Dissolution Inequality.}\\
By definition, the efficiency trap dissolves when the quantized model generates tokens faster than the native model, which is mathematically equivalent to $\mathrm{COR} < 0$, or $\tau_q < \tau_{\mathrm{ref}}$. Setting up the inequality:
\begin{equation}
    \frac{\alpha N p}{\mathcal{B}_W \cdot B} + \phi(h, p) < \frac{\alpha N \pi}{\mathcal{B}_W \cdot B}
\end{equation}
Subtracting the compressed memory latency from both sides isolates the casting overhead:
\begin{equation}
    \phi(h, p) < \frac{\alpha N \pi - \alpha N p}{\mathcal{B}_W \cdot B} = \frac{\alpha N (\pi - p)}{\mathcal{B}_W \cdot B}
\end{equation}
To find the scale threshold, we solve for $N$ by multiplying both sides by $\frac{\mathcal{B}_W \cdot B}{\alpha (\pi - p)}$:
\begin{equation}
    N > \frac{\phi(h,p) \cdot \mathcal{B}_W \cdot B}{\alpha (\pi - p)} \equiv N^*
\end{equation}
This rigorous bound proves that the Quantization Trap is not a hardware glitch, but a deterministic algebraic inequality. The trap dissolves if and only if the parameter count $N$ provides enough raw bandwidth savings $\alpha N (\pi - p)$ to offset the fixed translation tax $\phi(h,p)$.
\end{proof}


\begin{thebibliography}{100}


\bibitem{kaplan2020scaling}
Kaplan, J., McCandlish, S., Hernandez, D., Brown, T. B., Chess, B., Child, R., Gray, S., Radford, A., Wu, J., \& Amodei, D. (2020). Scaling laws for neural language models. \textit{arXiv preprint arXiv:2001.08361}.

\bibitem{dettmers2024qlora}
Dettmers, T., Pagnoni, A., Holtzman, A., \& Zettlemoyer, L. (2024). QLoRA: Efficient Finetuning of Quantized LLMs. \textit{Advances in Neural Information Processing Systems}, \textit{36}.

\bibitem{frantar2022gptq}
Frantar, E., Ashkboos, S., Hoefler, T., \& Alistarh, D. (2023). GPTQ: Accurate Post-Training Quantization for Generative Pre-trained Transformers. In \textit{International Conference on Learning Representations (ICLR)}.

\bibitem{lin2023awq}
Lin, J., Tang, J., Tang, H., Yang, S., Xiao, G., Dang, X., Gan, C.,  Han, S. (2024). AWQ: Activation-aware Weight Quantization for LLM Compression and Acceleration. In \textit{Proceedings of Machine Learning and Systems (MLSys)}.

\bibitem{wei2022}
Wei, J., Wang, X., Schuurmans, D., Maeda, M., Polozov, A., Xia, Y., Chi, E., Le, Q. V., \& Zhou, D. (2022). Chain-of-thought prompting elicits reasoning in large language models. \textit{Advances in Neural Information Processing Systems}, \textit{35}, 24824--24837.

\bibitem{mistral}
Jiang, A. Q., Sablayrolles, A., Mensch, A., Bamford, C., Chaplot, D. S., Casas, D. d. l., Bressand, F., Lengyel, G., Lample, G., Saulnier, L., \& others. (2023). Mistral 7B. \textit{arXiv preprint arXiv:2310.06825}.

\bibitem{marlin2024}
Frantar, E., Castro, R.~L., Chen, J., Hoefler, T., and Alistarh, D.
\newblock {MARLIN}: Mixed-Precision Auto-Regressive Parallel Inference on Large Language Models.
\newblock \emph{arXiv preprint arXiv:2408.11743}, 2024.

\bibitem{lin2024awq}
Lin, J., Tang, J., Tang, H., Yang, S., Xiao, G., Dang, X., Gan, C., and Han, S.
\newblock {AWQ}: Activation-aware Weight Quantization for {LLM} Compression and Acceleration.
\newblock In \emph{Proceedings of Machine Learning and Systems (MLSys)}, 2024.

\bibitem{tseng2024quip}
Tseng, A., Chee, J., Sun, Q., Kuleshov, V., and De~Sa, C.
\newblock {QuIP\#}: Even Better {LLM} Quantization with {Hadamard} Incoherence and Lattice Codebooks.
\newblock In \emph{International Conference on Machine Learning (ICML)}, pp.~48630--48656, 2024.

\bibitem{sheng2023flexgen}
Sheng, Y., Zheng, L., Yuan, B., Li, Z., Ryabinin, M., Chen, B., Liang, P., R\'{e}, C., Stoica, I., and Zhang, C.
\newblock {FlexGen}: High-Throughput Generative Inference of Large Language Models with a Single {GPU}.
\newblock In \emph{International Conference on Machine Learning (ICML)}, 2023.


\bibitem{gsm8k}
Cobbe, K., Kosaraju, V., Bavarian, M., Chen, M., Jun, H., Kaiser, L., Plappert, M., Tworek, J., Hilton, J., Nakano, R., \& others. (2021). Training verifiers to solve math word problems. \textit{arXiv preprint arXiv:2110.14168}.

\bibitem{cobbe2021}
Cobbe, K., Kosaraju, V., Bavarian, M., Chen, M., Jun, H., Kaiser, L., Plappert, M., Tworek, J., Hilton, J., Nakano, R., \& others. (2021). Training verifiers to solve math word problems. \textit{arXiv preprint arXiv:2110.14168}.


\bibitem{mathqa}
Amini, A., Gabriel, S., Lin, S., Koncel-Kedziorski, R., Choi, Y., \& Hajishirzi, H. (2019). MathQA: Towards Interpretable Math Word Problem Solving with Operation-Based Formalisms. In \textit{Proceedings of the 2019 Conference of the North American Chapter of the Association for Computational Linguistics}.

\bibitem{hoffmann2022chinchilla}
Hoffmann, J., Borgeaud, S., Mensch, A., Buchatskaya, E., Cai, T., Rutherford, E., Casas, D. d. l., Hendricks, L. A., Welbl, J., Clark, A., \& others. (2022). Training Compute-Optimal Large Language Models. \textit{arXiv preprint arXiv:2203.15556}.



\bibitem{hendrycks2020}
Hendrycks, D., Burns, C., Basart, S., Zou, A., Mazeika, M., Song, D., \& Steinhardt, J. (2021). Measuring Massive Multitask Language Understanding. In \textit{International Conference on Learning Representations (ICLR)}.


\bibitem{zhao2024atom}
Zhao, Y., Lin, C., \& others. (2023). Atom: Low-bit Quantization for Efficient and Accurate LLM Serving. \textit{arXiv preprint arXiv:2310.19102}.

\bibitem{liu2025quantization}
R. Liu, Y. Sun, M. Zhang, H. Bai, X. Yu, T. Yu, C. Yuan, and L. Hou,
``Quantization Hurts Reasoning? An Empirical Study on Quantized Reasoning Models,''
\textit{COLM}, 2025.

\bibitem{li2025quantmeets}
Z. Li, Y. Su, R. Yang, C. Xie, Z. Wang, Z. Xie, N. Wong, and H. Yang,
``Quantization Meets Reasoning: Exploring LLM Low-Bit Quantization Degradation for Mathematical Reasoning,''
\textit{arXiv preprint arXiv:2501.03035}, 2025.

\bibitem{nvidia2020ampere}
NVIDIA. (2020). NVIDIA A100 Tensor Core GPU Architecture. NVIDIA Corporation.


\bibitem{ahmed2020compute_divide}
Ahmed, N., \& Wahed, M. (2020). The De-democratization of AI: Deep Learning and the Compute Divide in Artificial Intelligence Research. \textit{arXiv preprint arXiv:2010.15581}.




\bibitem{nvml}
NVIDIA. (2024). NVIDIA Management Library (NVML) Reference Guide. NVIDIA Corporation.



\bibitem{huang2020evaluating}
Huang, H., \& others. (2020). Analysis of GPU Power Consumption Using Internal Sensors. In \textit{Proceedings of the 2020 IEEE International Parallel and Distributed Processing Symposium}.

\bibitem{henderson2020towards}
Henderson, P., Hu, J., Romoff, J., Brunskill, E., Jurafsky, D., \& Pineau, J. (2020). Towards the Systematic Reporting of the Energy and Carbon Footprints of Machine Learning. \textit{Journal of Machine Learning Research}, \textit{21}.

\bibitem{kumbhare2020gpu}
Kumbhare, N., \& others. (2020). Understanding GPU Power: A Survey of Profiling, Modeling, and Simulation Methods. \textit{ACM Computing Surveys}.


\bibitem{pope2023}
Pope, R., Douglas, S., Chowdhery, A., Devlin, J., Bradbury, J., Heek, J., Xiao, K., Agrawal, S., \& Dean, J. (2023). Efficiently Scaling Transformer Inference. In \textit{Proceedings of the Sixth Conference on Machine Learning and Systems (MLSys)}.



\bibitem{schmidt2021}
Schmidt, V., Goyal, K., Joshi, A., Feld, B., Connell, L., Laskaris, N., Blank, D., \& others. (2021). CodeCarbon: Estimate and Track Carbon Emissions from Machine Learning Computing.


\bibitem{ho2020}
Holtzman, A., Buys, J., Du, L., Forbes, M., \& Choi, Y. (2020). The Curious Case of Neural Text Degeneration. In \textit{International Conference on Learning Representations (ICLR)}.

\bibitem{zhang2023}
Zhang, M., Press, O., Merrill, W., Liu, A., \& Smith, N. A. (2023). How language model hallucinations can snowball. \textit{arXiv preprint arXiv:2305.13534}.


\bibitem{falcon3}
Falcon-LLM Team. (2024). \textit{The Falcon 3 Family of Open Models}. Technology Innovation Institute. [Online]. Available: \url{https://huggingface.co/blog/falcon3}

\bibitem{Ra2024}
S.~Rajput and T.~Sharma, ``Benchmarking Emerging Deep Learning Quantization Methods for Energy Efficiency,'' \emph{ICSA-C}, 2024.

\bibitem{Hu2025}
E.~Husom \emph{et al.}, ``Sustainable LLM Inference for Edge AI: Evaluating Quantized LLMs for Energy Efficiency, Output Accuracy, and Inference Latency,'' \emph{TioT}, 2025.

\bibitem{nvidia2024b}
NVIDIA. (2024). NVIDIA Blackwell Architecture Whitepaper. Retrieved from https://www.nvidia.com/en-us/data-center/blackwell-architecture/


\end{thebibliography}
\end{document}